\documentclass[a4paper]{llncs}

\usepackage{amsfonts}
\usepackage{graphicx}
\usepackage{caption}
\usepackage{subcaption}
\usepackage{footnote}
\usepackage{enumerate}

\RequirePackage{graphics}
\usepackage{xcolor}
\usepackage{footnote}
\usepackage{xparse}
\usepackage{cite}
\usepackage{booktabs}

\usepackage[noend]{algpseudocode}
\usepackage{algorithm}

\algblock{ParFor}{EndParFor}
\algnewcommand\algorithmicparfor{\textbf{parfor}}
\algnewcommand\algorithmicpardo{\textbf{do}}
\algnewcommand\algorithmicendparfor{\textbf{end\ parfor}}
\algrenewtext{ParFor}[1]{\algorithmicparfor\ #1\ \algorithmicpardo}
\algrenewtext{EndParFor}{\algorithmicendparfor}
\algnewcommand\algorithmicinput{\textbf{Input:}}
\algnewcommand\algorithmicoutput{\textbf{Output:}}
\algnewcommand\Input{\item[\algorithmicinput]}%
\algnewcommand\Output{\item[\algorithmicoutput]}%

\usepackage[hyphens]{url} 
\usepackage{amssymb}
\usepackage{amsmath}
\usepackage{keyval}
\usepackage{xspace}
\usepackage{paralist}
\usepackage{listings}
\usepackage{multirow}
\usepackage{stmaryrd}
\usepackage{adjustbox} 
\usepackage{makecell}
\usepackage{sidecap}
\usepackage{booktabs} 
\usepackage{threeparttable, tablefootnote}

\RequirePackage{tikz}
\usetikzlibrary{arrows,automata,shapes,calc,through,decorations.pathmorphing,decorations.fractals,chains,shapes.multipart}

\usepackage{arydshln} 
\usepackage[free-standing-units]{siunitx}
\usepackage{circuitikz}
\usepackage{tabularx} 

\def\titlename{Verification of Deep Convolutional Neural Networks Using ImageStars  \xspace}
\def\authortran{Hoang-Dung Tran}

\def\authorweiming{Weiming Xiang}
\def\authorstan{Stanley Bak}

\def\authortaylor{Taylor T. Johnson}

\usepackage[pdftex]{hyperref}
\hypersetup{
  colorlinks				=true,
  citecolor					={purple},
  linkcolor 				={blue},
  bookmarksnumbered	=true
}

\newcommand{\commenttaylor}[1]{}

\newcommand{\nnnum}[1]{\relax\ifmmode
  {\mathbb #1}_{\geq 0} \else ${\mathbb #1}_{\geq 0}$
  \fi}
\newcommand{\npnum}[1]{\relax\ifmmode
  {\mathbb #1}_{\leq 0} \else ${\mathbb #1}_{\leq 0}$
  \fi}
\newcommand{\pnum}[1]{\relax\ifmmode
  {\mathbb #1}_{> 0} \else ${\mathbb #1}_{> 0}$
  \fi}
\newcommand{\nnum}[1]{\relax\ifmmode
  {\mathbb #1}_{< 0} \else ${\mathbb #1}_{< 0}$
  \fi}
\newcommand{\plnum}[1]{\relax\ifmmode
  {\mathbb #1}_{+} \else ${\mathbb #1}_{+}$
  \fi}
\newcommand{\nenum}[1]{\relax\ifmmode
  {\mathbb #1}_{-} \else ${\mathbb #1}_{-}$
  \fi}

\newcommand{\extb}[1]{\relax\ifmmode {\sf ExtBeh}_{#1} \else ${\sf ExtBeh}_{#1}$\fi}
\newcommand{\tdists}[1]{\relax\ifmmode {\sf Tdists}_{#1} \else ${\sf Tdists}_{#1}$\fi}

\newcommand{\exec}[1]{\relax\ifmmode {\sf Execs}_{#1} \else ${\sf Exec}_{#1}$\fi}
\newcommand{\execf}[1]{\relax\ifmmode {\sf Execs}^*_{#1} \else ${\sf Exec}^*_{#1}$\fi}
\newcommand{\execi}[1]{\relax\ifmmode {\sf Execs}^\omega_{#1} \else ${\sf Exec}^\omega_{#1}$\fi}

\newcommand{\ctrace}[1]{\relax\ifmmode {\sf Ctraces}_{#1} \else ${\sf Ctraces}_{#1}$\fi}

\newcommand{\trace}[1]{\relax\ifmmode {\sf Traces}_{#1} \else ${\sf Traces}_{#1}$\fi}
\newcommand{\tracef}[1]{\relax\ifmmode {\sf Traces}^*_{#1} \else ${\sf Traces}^*_{#1}$\fi}
\newcommand{\tracei}[1]{\relax\ifmmode {\sf Traces}^\omega_{#1} \else ${\sf Traces}^\omega_{#1}$\fi}

\newcommand{\frag}[1]{\relax\ifmmode {\sf Frags}_{#1} \else ${\sf Frags}_{#1}$\fi}
\newcommand{\fragf}[1]{\relax\ifmmode {\sf Frags}^*_{#1} \else ${\sf Frags}^*_{#1}$\fi}
\newcommand{\fragi}[1]{\relax\ifmmode {\sf Frags}^\omega_{#1} \else ${\sf Frags}^\omega_{#1}$\fi}

\newcommand{\reach}[1]{\relax\ifmmode {\sf Reach}_{#1} \else ${\sf Reach}_{#1}$\fi}

\def\A{{\cal A}} 
\def\E{{\cal E}} 
\def\I{{\cal I}} 
\def\R{{\cal R}} 
\def\T{{\cal T}} 
\def\U{{\cal U}} 

\newcommand{\col}[1]{\relax\ifmmode \mathscr #1\else $\mathscr #1$\fi}

\definecolor{HIOAcolor}{rgb}{0.776,0.22,0.07}

\newcommand{\SC}[2]{\relax\ifmmode {\tt Scount}(#1,#2) \else ${\tt Scount}(#1,#2)$\fi}
\newcommand{\SCM}[2]{\relax\ifmmode {\tt Smin}(#1,#2) \else ${\tt Smin}(#1,#2)$\fi}
\newcommand{\Aut}[1]{\relax\ifmmode {\tt Aut}(#1) \else ${\tt Aut}(#1)$\fi}

\newcommand{\act}[1]{{\operatorname{\mathsf{#1}}}}

\renewcommand{\eqref}[1]{Equation~\ref{eq:#1}}

\newcommand{\remove}[1]{}
\newcommand{\salg}[1]{\relax\ifmmode {\mathcal F}_{#1}\else ${\mathcal F}_{#1}$\fi}
\newcommand{\msp}[1]{\relax\ifmmode (#1, \salg{#1}) \else $(#1, \salg{#1})$\fi}
\newcommand{\msprod}[2]{\relax\ifmmode ( #1 \times #2, \salg{#1} \otimes \salg{#2}) \else $(#1 \times #2, \salg{#1} \otimes \salg{#2})$\fi}
\newcommand{\dist}[1]{\relax\ifmmode {\mathcal P}\msp{#1}
  \else ${\mathcal P}\msp{#1}$\fi}
\newcommand{\subdist}[1]{\relax\ifmmode {\mathcal S}{\mathcal P}\msp{#1}
  \else ${\mathcal S}{\mathcal P}\msp{#1}$\fi}
\newcommand{\disc}[1]{\relax\ifmmode {\sf Disc}(#1)
  \else ${\sf Disc}(#1)$\fi}

\newcommand{\Trajeq}{\relax\ifmmode {\mathcal R}_\T \else ${\mathcal R}_\T$\fi}
\newcommand{\Acteq}{\relax\ifmmode {\mathcal R}_A \else ${\mathcal R}_A$\fi}
\newcommand{\noop}{\relax\ifmmode \lambda \else $\lambda$\fi}
\newcommand{\close}[1]{\relax\ifmmode \overline{#1} \else $\overline{#1}$\fi}

\newcommand{\tup}[1]
           {
             \relax\ifmmode
             \langle #1 \rangle
             \else $\langle$ #1 $\rangle$ \fi
           }

\newcommand{\lit}[1]{ \relax\ifmmode
                \mathord{\mathcode`\-="702D\sf #1\mathcode`\-="2200}
                \else {\it #1} \fi }

\newcommand{\figuresize}{\scriptsize}

\lstdefinelanguage{ioa}{
  basicstyle=\figuresize,
  keywordstyle=\bf \figuresize,
  identifierstyle=\it \figuresize,
  emphstyle=\tt \figuresize,
  mathescape=true,
  tabsize=20,
  sensitive=false,
  columns=fullflexible,
  keepspaces=false,
  flexiblecolumns=true,
  basewidth=0.05em,
  escapeinside={(*@}{@*)},
  moredelim=[il][\rm]{//},
  moredelim=[is][\sf \figuresize]{!}{!},
  moredelim=[is][\bf \figuresize]{*}{*},
  keywords={automaton,and,
  	 choose,const,continue, components,
  	 discrete, do,
  	 eff, Eff, external,else, elseif, evolve, end,
  	 fi,for, forward, from,
  	 hidden,
  	 in,input,internal,if,invariant, initially, imports,
     let,
     or, output, operators, od, of,
     pre, Pre,
     return,
     such,satisfies, stop, signature, simulation,
     trajectories,trajdef, transitions, that,then, type, types, to, tasks,
     variables, vocabulary,
     when,where, with,while},
  emph={set, seq, tuple, map, array, enumeration},
   literate=
        {(}{{$($}}1
        {)}{{$)$}}1
        {\\in}{{$\in\ $}}1
        {\\preceq}{{$\preceq\ $}}1
        {\\subset}{{$\subset\ $}}1
        {\\subseteq}{{$\subseteq\ $}}1
        {\\supset}{{$\supset\ $}}1
        {\\supseteq}{{$\supseteq\ $}}1
        {\\forall}{{$\forall$}}1
        {\\le}{{$\le\ $}}1
        {\\ge}{{$\ge\ $}}1
        {\\gets}{{$\gets\ $}}1
        {\\cup}{{$\cup\ $}}1
        {\\cap}{{$\cap\ $}}1
        {\\langle}{{$\langle$}}1
        {\\rangle}{{$\rangle$}}1
        {\\exists}{{$\exists\ $}}1
        {\\bot}{{$\bot$}}1
        {\\rip}{{$\rip$}}1
        {\\emptyset}{{$\emptyset$}}1
        {\\notin}{{$\notin\ $}}1
        {\\not\\exists}{{$\not\exists\ $}}1
        {\\ne}{{$\ne\ $}}1
        {\\to}{{$\to\ $}}1
        {\\implies}{{$\implies\ $}}1
        {<}{{$<\ $}}1
        {>}{{$>\ $}}1
        {=}{{$=\ $}}1
        {~}{{$\neg\ $}}1
        {|}{{$\mid$}}1
        {'}{{$^\prime$}}1
        {\\A}{{$\forall\ $}}1
        {\\E}{{$\exists\ $}}1
        {\\nE}{{$\nexists\ $}}1
        {\\/}{{$\vee\,$}}1
        {\\vee}{{$\vee\,$}}1
        {/\\}{{$\wedge\,$}}1
        {\\wedge}{{$\wedge\,$}}1
        {=>}{{$\Rightarrow\ $}}1
        {->}{{$\rightarrow\ $}}1
        {<=}{{$\Leftarrow\ $}}1
        {<-}{{$\leftarrow\ $}}1
        {~=}{{$\neq\ $}}1
        {\\U}{{$\cup\ $}}1
        {\\I}{{$\cap\ $}}1
        {|-}{{$\vdash\ $}}1
        {-|}{{$\dashv\ $}}1
        {<<}{{$\ll\ $}}2
        {>>}{{$\gg\ $}}2
        {||}{{$\|$}}1
        {[}{{$[$}}1
        {]}{{$\,]$}}1
        {[[}{{$\langle$}}1
        {]]]}{{$]\rangle$}}1
        {]]}{{$\rangle$}}1
        {<=>}{{$\Leftrightarrow\ $}}2
        {<->}{{$\leftrightarrow\ $}}2
        {(+)}{{$\oplus\ $}}1
        {(-)}{{$\ominus\ $}}1
        {_i}{{$_{i}$}}1
        {_j}{{$_{j}$}}1
        {_{i,j}}{{$_{i,j}$}}3
        {_{j,i}}{{$_{j,i}$}}3
        {_0}{{$_0$}}1
        {_1}{{$_1$}}1
        {_2}{{$_2$}}1
        {_n}{{$_n$}}1
        {_p}{{$_p$}}1
        {_k}{{$_n$}}1
        {-}{{$\ms{-}$}}1
        {@}{{}}0
        {\\delta}{{$\delta$}}1
        {\\R}{{$\R$}}1
        {\\Rplus}{{$\Rplus$}}1
        {\\N}{{$\N$}}1
        {\\times}{{$\times\ $}}1
        {\\tau}{{$\tau$}}1
        {\\alpha}{{$\alpha$}}1
        {\\beta}{{$\beta$}}1
        {\\gamma}{{$\gamma$}}1
        {\\ell}{{$\ell\ $}}1
        {--}{{$-\ $}}1
        {\\TT}{{\hspace{1.5em}}}3
      }

\lstdefinelanguage{ioaNums}[]{ioa}
{
  numbers=left,
  numberstyle=\tiny,
  stepnumber=2,
  numbersep=4pt
}

\lstdefinelanguage{ioaNumsRight}[]{ioa}
{
  numbers=right,
  numberstyle=\tiny,
  stepnumber=2,
  numbersep=4pt
}

\lstnewenvironment{IOA}%
  {\lstset{language=IOA}}
  {}

\lstnewenvironment{IOANums}%
  {
  \if@firstcolumn
    \lstset{language=IOA, numbers=left, firstnumber=auto}
  \else
    \lstset{language=IOA, numbers=right, firstnumber=auto}
  \fi
  }
  {}

\lstnewenvironment{IOANumsRight}%
  {
    \lstset{language=IOA, numbers=right, firstnumber=auto}
  }
  {}

\newcommand{\linefigioa}[9]{

}

\lstdefinelanguage{ioaLang}{%
  basicstyle=\ttfamily\small,
  keywordstyle=\rmfamily\bfseries\small,
  identifierstyle=\small,
  keywords={assumes,automaton,axioms,backward,bounds,by,case,choose,components,const,d,det,discrete,do,eff,else,elseif,ensuring,enumeration,evolve,fi,fire,follow,for,forward,from,hidden,if,in,%
    input,initially,internal,invariant,let, local,od,of,output,pre,schedule,signature,so,%
    simulation,states,variables, tasks, stop,tasks,that,then,to,trajdef,trajectory,trajectories,transitions,tuple,type,union,urgent,uses,when,where,while,yield},
  literate=
        {\\in}{{$\in$}}1
        {\\preceq}{{$\preceq$}}1
        {\\subset}{{$\subset$}}1
        {\\subseteq}{{$\subseteq$}}1
        {\\supset}{{$\supset$}}1
        {\\supseteq}{{$\supseteq$}}1
        {\\rho}{{$\rho$}}1
        {\\infty}{{$\infty$}}1
        {<}{{$<$}}1
        {>}{{$>$}}1
        {=}{{$=$}}1
        {~}{{$\neg$}}1
        {|}{{$\mid$}}1
        {'}{{$^\prime$}}1
        {\\A}{{$\forall$}}1 {\\E}{{$\exists$}}1
        {\\/}{{$\vee$}}1 {/\\}{{$\wedge$}}1
        {=>}{{$\Rightarrow$}}1
        {->}{{$\rightarrow$}}1
        {<=}{{$\leq$}}1 {>=}{{$\geq$}}1 {~=}{{$\neq$}}1
        {\\U}{{$\cup$}}1 {\\I}{{$\cap$}}1
        {|-}{{$\vdash$}}1 {-|}{{$\dashv$}}1
        {<<}{{$\ll$}}2 {>>}{{$\gg$}}2
        {||}{{$\|$}}1
        {<=>}{{$\Leftrightarrow$}}2
        {<->}{{$\leftrightarrow$}}2
        {(+)}{{$\oplus$}}1
        {(-)}{{$\ominus$}}1
}

\lstdefinelanguage{bigIOALang}{%
  basicstyle=\ttfamily,
  keywordstyle=\rmfamily\bfseries,
  identifierstyle=,
  keywords={assumes,automaton,axioms,backward,by,case,choose,components,const,%
    d,det,discrete,do,eff,else,elseif,ensuring,enumeration,evolve,fi,for,forward,from,hidden,if,in%
    input,initially,internal,invariant,local,od,of,output,pre,schedule,signature,so,%
    tasks, simulation,states,stop,tasks,that,then,to,trajdef,trajectories,transitions,tuple,type,union,urgent,uses,when,where,yield},
  literate=
        {\\in}{{$\in$}}1
        {\\preceq}{{$\preceq$}}1
        {\\subset}{{$\subset$}}1
        {\\subseteq}{{$\subseteq$}}1
        {\\supset}{{$\supset$}}1
        {\\supseteq}{{$\supseteq$}}1
        {<}{{$<$}}1
        {>}{{$>$}}1
        {=}{{$=$}}1
        {~}{{$\neg$}}1
        {|}{{$\mid$}}1
        {'}{{$^\prime$}}1
        {\\A}{{$\forall$}}1 {\\E}{{$\exists$}}1
        {\\/}{{$\vee$}}1 {/\\}{{$\wedge$}}1
        {=>}{{$\Rightarrow$}}1
        {->}{{$\rightarrow$}}1
        {<=}{{$\leq$}}1 {>=}{{$\geq$}}1 {~=}{{$\neq$}}1
        {\\U}{{$\cup$}}1 {\\I}{{$\cap$}}1
        {|-}{{$\vdash$}}1 {-|}{{$\dashv$}}1
        {<<}{{$\ll$}}2 {>>}{{$\gg$}}2
        {||}{{$\|$}}1
        {<=>}{{$\Leftrightarrow$}}2
        {<->}{{$\leftrightarrow$}}2
        {(+)}{{$\oplus$}}1
        {(-)}{{$\ominus$}}1
}

\lstnewenvironment{BigIOA}%
  {\lstset{language=bigIOALang,basicstyle=\ttfamily}
   \csname lst@SetFirstLabel\endcsname}
  {\csname lst@SaveFirstLabel\endcsname\vspace{-4pt}\noindent}

\lstnewenvironment{SmallIOA}%
  {\lstset{language=ioaLang,basicstyle=\ttfamily\scriptsize}
   \csname lst@SetFirstLabel\endcsname}
  {\csname lst@SaveFirstLabel\endcsname\noindent}

\newlength{\bracklen}

\newcommand{\tri}[3]{\ensuremath{\mathit{#1}^\mathit{#2}_\mathit{#3}}}

\newcommand{\sugLocalVars}[2]{\ifthenelse{\equal{}{#2}}%
                             {\tri{localVars}{#1}{desug}}%
                             {\tri{localVars}{#1}{#2,desug}}}
\newcommand{\sugVars}[2]{\ifthenelse{\equal{}{#2}}%
                        {\tri{vars}{#1}{desug}}%
                        {\tri{vars}{#1}{#2,desug}}}

\newenvironment{subSyntax}{\begin{array}{l}}{\end{array}}

\newcommand{\ms}[1]{\ifmmode%
\mathord{\mathcode`-="702D\it #1\mathcode`\-="2200}\else%
$\mathord{\mathcode`-="702D\it #1\mathcode`\-="2200}$\fi}

\def\A{{\cal A}} 
\def\T{{\cal T}} 

\lstdefinelanguage{pvs}{
  basicstyle=\tt \figuresize,
  keywordstyle=\sc \figuresize,
  identifierstyle=\it \figuresize,
  emphstyle=\tt \figuresize,
  mathescape=true,
  tabsize=20,
  sensitive=false,
  columns=fullflexible,
  keepspaces=false,
  flexiblecolumns=true,
  basewidth=0.05em,
  moredelim=[il][\rm]{//},
  moredelim=[is][\sf \figuresize]{!}{!},
  moredelim=[is][\bf \figuresize]{*}{*},
  keywords={and,
  	 begin,
  	 cases, const,
  	 do,
  	 external, else, exists, end, endcases, endif,
  	 fi,for, forall, from,
  	 hidden,
  	 in, if, importing,
     let, lambda, lemma,
     measure,
     not,
     or, of,
     return, recursive,
     stop,
     theory, that,then, type, types, type+, to, theorem,
     var,
     with,while},
  emph={nat, setof, sequence, eq, tuple, map, array, enumeration, bool, real, exp, nnreal, posreal},
   literate=
        {(}{{$($}}1
        {)}{{$)$}}1
        {\\in}{{$\in\ $}}1
        {\\mapsto}{{$\rightarrow\ $}}1
        {\\preceq}{{$\preceq\ $}}1
        {\\subset}{{$\subset\ $}}1
        {\\subseteq}{{$\subseteq\ $}}1
        {\\supset}{{$\supset\ $}}1
        {\\supseteq}{{$\supseteq\ $}}1
        {\\forall}{{$\forall$}}1
        {\\le}{{$\le\ $}}1
        {\\ge}{{$\ge\ $}}1
        {\\gets}{{$\gets\ $}}1
        {\\cup}{{$\cup\ $}}1
        {\\cap}{{$\cap\ $}}1
        {\\langle}{{$\langle$}}1
        {\\rangle}{{$\rangle$}}1
        {\\exists}{{$\exists\ $}}1
        {\\bot}{{$\bot$}}1
        {\\rip}{{$\rip$}}1
        {\\emptyset}{{$\emptyset$}}1
        {\\notin}{{$\notin\ $}}1
        {\\not\\exists}{{$\not\exists\ $}}1
        {\\ne}{{$\ne\ $}}1
        {\\to}{{$\to\ $}}1
        {\\implies}{{$\implies\ $}}1
        {<}{{$<\ $}}1
        {>}{{$>\ $}}1
        {=}{{$=\ $}}1
        {~}{{$\neg\ $}}1
        {|}{{$\mid$}}1
        {'}{{$^\prime$}}1
        {\\A}{{$\forall\ $}}1
        {\\E}{{$\exists\ $}}1
        {\\/}{{$\vee\,$}}1
        {\\vee}{{$\vee\,$}}1
        {/\\}{{$\wedge\,$}}1
        {\\wedge}{{$\wedge\,$}}1
        {->}{{$\rightarrow\ $}}1
        {=>}{{$\Rightarrow\ $}}1
        {->}{{$\rightarrow\ $}}1
        {<=}{{$\Leftarrow\ $}}1
        {<-}{{$\leftarrow\ $}}1
        {~=}{{$\neq\ $}}1
        {\\U}{{$\cup\ $}}1
        {\\I}{{$\cap\ $}}1
        {|-}{{$\vdash\ $}}1
        {-|}{{$\dashv\ $}}1
        {<<}{{$\ll\ $}}2
        {>>}{{$\gg\ $}}2
        {||}{{$\|$}}1
        {[}{{$[$}}1
        {]}{{$\,]$}}1
        {[[}{{$\langle$}}1
        {]]]}{{$]\rangle$}}1
        {]]}{{$\rangle$}}1
        {<=>}{{$\Leftrightarrow\ $}}2
        {<->}{{$\leftrightarrow\ $}}2
        {(+)}{{$\oplus\ $}}1
        {(-)}{{$\ominus\ $}}1
        {_i}{{$_{i}$}}1
        {_j}{{$_{j}$}}1
        {_{i,j}}{{$_{i,j}$}}3
        {_{j,i}}{{$_{j,i}$}}3
        {_0}{{$_0$}}1
        {_1}{{$_1$}}1
        {_2}{{$_2$}}1
        {_n}{{$_n$}}1
        {_p}{{$_p$}}1
        {_k}{{$_n$}}1
        {-}{{$\ms{-}$}}1
        {@}{{}}0
        {\\delta}{{$\delta$}}1
        {\\R}{{$\R$}}1
        {\\Rplus}{{$\Rplus$}}1
        {\\N}{{$\N$}}1
        {\\times}{{$\times\ $}}1
        {\\tau}{{$\tau$}}1
        {\\alpha}{{$\alpha$}}1
        {\\beta}{{$\beta$}}1
        {\\gamma}{{$\gamma$}}1
        {\\ell}{{$\ell\ $}}1
        {--}{{$-\ $}}1
        {\\TT}{{\hspace{1.5em}}}3
      }

\lstdefinelanguage{BigPVS}{
  basicstyle=\tt,
  keywordstyle=\sc,
  identifierstyle=\it,
  emphstyle=\tt ,
  mathescape=true,
  tabsize=20,
  sensitive=false,
  columns=fullflexible,
  keepspaces=false,
  flexiblecolumns=true,
  basewidth=0.05em,
  moredelim=[il][\rm]{//},
  moredelim=[is][\sf \figuresize]{!}{!},
  moredelim=[is][\bf \figuresize]{*}{*},
  keywords={and,
  	 begin,
  	 cases, const,
  	 do, datatype,
  	 external, else, exists, end, endif, endcases,
  	 fi,for, forall, from,
  	 hidden,
  	 in, if, importing,
     let, lambda, lemma,
     measure,
     not,
     or, of,
     return, recursive,
     stop,
     theory, that,then, type, types, type+, to, theorem,
     var,
     with,while},
  emph={nat, setof, sequence, eq, tuple, map, array, first, rest, add, enumeration, bool, real, posreal, nnreal},
   literate=
        {(}{{$($}}1
        {)}{{$)$}}1
        {\\in}{{$\in\ $}}1
        {\\mapsto}{{$\rightarrow\ $}}1
        {\\preceq}{{$\preceq\ $}}1
        {\\subset}{{$\subset\ $}}1
        {\\subseteq}{{$\subseteq\ $}}1
        {\\supset}{{$\supset\ $}}1
        {\\supseteq}{{$\supseteq\ $}}1
        {\\forall}{{$\forall$}}1
        {\\le}{{$\le\ $}}1
        {\\ge}{{$\ge\ $}}1
        {\\gets}{{$\gets\ $}}1
        {\\cup}{{$\cup\ $}}1
        {\\cap}{{$\cap\ $}}1
        {\\langle}{{$\langle$}}1
        {\\rangle}{{$\rangle$}}1
        {\\exists}{{$\exists\ $}}1
        {\\bot}{{$\bot$}}1
        {\\rip}{{$\rip$}}1
        {\\emptyset}{{$\emptyset$}}1
        {\\notin}{{$\notin\ $}}1
        {\\not\\exists}{{$\not\exists\ $}}1
        {\\ne}{{$\ne\ $}}1
        {\\to}{{$\to\ $}}1
        {\\implies}{{$\implies\ $}}1
        {<}{{$<\ $}}1
        {>}{{$>\ $}}1
        {=}{{$=\ $}}1
        {~}{{$\neg\ $}}1
        {|}{{$\mid$}}1
        {'}{{$^\prime$}}1
        {\\A}{{$\forall\ $}}1
        {\\E}{{$\exists\ $}}1
        {\\/}{{$\vee\,$}}1
        {\\vee}{{$\vee\,$}}1
        {/\\}{{$\wedge\,$}}1
        {\\wedge}{{$\wedge\,$}}1
        {->}{{$\rightarrow\ $}}1
        {=>}{{$\Rightarrow\ $}}1
        {->}{{$\rightarrow\ $}}1
        {<=}{{$\Leftarrow\ $}}1
        {<-}{{$\leftarrow\ $}}1
        {~=}{{$\neq\ $}}1
        {\\U}{{$\cup\ $}}1
        {\\I}{{$\cap\ $}}1
        {|-}{{$\vdash\ $}}1
        {-|}{{$\dashv\ $}}1
        {<<}{{$\ll\ $}}2
        {>>}{{$\gg\ $}}2
        {||}{{$\|$}}1
        {[}{{$[$}}1
        {]}{{$\,]$}}1
        {[[}{{$\langle$}}1
        {]]]}{{$]\rangle$}}1
        {]]}{{$\rangle$}}1
        {<=>}{{$\Leftrightarrow\ $}}2
        {<->}{{$\leftrightarrow\ $}}2
        {(+)}{{$\oplus\ $}}1
        {(-)}{{$\ominus\ $}}1
        {_i}{{$_{i}$}}1
        {_j}{{$_{j}$}}1
        {_{i,j}}{{$_{i,j}$}}3
        {_{j,i}}{{$_{j,i}$}}3
        {_0}{{$_0$}}1
        {_1}{{$_1$}}1
        {_2}{{$_2$}}1
        {_n}{{$_n$}}1
        {_p}{{$_p$}}1
        {_k}{{$_n$}}1
        {-}{{$\ms{-}$}}1
        {@}{{}}0
        {\\delta}{{$\delta$}}1
        {\\R}{{$\R$}}1
        {\\Rplus}{{$\Rplus$}}1
        {\\N}{{$\N$}}1
        {\\times}{{$\times\ $}}1
        {\\tau}{{$\tau$}}1
        {\\alpha}{{$\alpha$}}1
        {\\beta}{{$\beta$}}1
        {\\gamma}{{$\gamma$}}1
        {\\ell}{{$\ell\ $}}1
        {--}{{$-\ $}}1
        {\\TT}{{\hspace{1.5em}}}3
      }

\lstdefinelanguage{pvsNums}[]{pvs}
{
  numbers=left,
  numberstyle=\tiny,
  stepnumber=2,
  numbersep=4pt
}

\lstdefinelanguage{pvsNumsRight}[]{pvs}
{
  numbers=right,
  numberstyle=\tiny,
  stepnumber=2,
  numbersep=4pt
}

\lstnewenvironment{BigPVS}%
  {\lstset{language=BigPVS}}
  {}

\lstnewenvironment{PVSNums}%
  {
  \if@firstcolumn
    \lstset{language=pvs, numbers=left, firstnumber=auto}
  \else
    \lstset{language=pvs, numbers=right, firstnumber=auto}
  \fi
  }
  {}

\lstnewenvironment{PVSNumsRight}%
  {
    \lstset{language=pvs, numbers=right, firstnumber=auto}
  }
  {}

\newcommand{\linefigpvs}[9]{

}

\lstdefinelanguage{pvsproof}{
  basicstyle=\tt \figuresize,
  mathescape=true,
  tabsize=4,
  sensitive=false,
  columns=fullflexible,
  keepspaces=false,
  flexiblecolumns=true,
  basewidth=0.05em,
}

\def\N{\act{N}}

\newcommand{\localvar}[2]{{{#1_{#2}}}}

\def\xi{\localvar{x}{i}}

\def\reach{{\sf Reach}}

\def\Xi{\mathit{X_i}}

\begin{document}
%

\title{\titlename}

%
%
%

\author{\authortran \inst{1, 2}, \authorstan \inst{3}, \authorweiming \inst{4}, \and \authortaylor \inst{2}}

\authorrunning{Tran et al.}
%


\institute{Department of Computer Science and Engineering, University of Nebraska, USA \and
Department of Electrical Engineering and Computer Science, Vanderbilt University, USA \and
Department of Computer Science, Stony Brook University, USA \and
School of Computer and Cyber Sciences, Augusta University, USA}

%
%
\maketitle


\begin{abstract}
Convolutional Neural Networks (CNN) have redefined state-of-the-art in many real-world applications, such as facial recognition, image classification, human pose estimation, and semantic segmentation. Despite their success, CNNs are vulnerable to adversarial attacks, where slight changes to their inputs may lead to sharp changes in their output in even well-trained networks. Set-based analysis methods can detect or prove the absence of bounded adversarial attacks, which can then be used to evaluate the effectiveness of neural network training methodology. Unfortunately, existing verification approaches have limited scalability in terms of the size of networks that can be analyzed.

In this paper, we describe a set-based framework that successfully deals with real-world CNNs, such as VGG16 and VGG19, that have high accuracy on ImageNet. Our approach is based on a new set representation called the ImageStar, which enables efficient exact and over-approximative analysis of CNNs. ImageStars perform efficient set-based analysis by combining operations on concrete images with linear programming (LP). Our approach is implemented in a tool called NNV, and can verify the robustness of VGG networks with respect to a small set of input states, derived from adversarial attacks, such as the DeepFool attack. The experimental results show that our approach is less conservative and faster than existing zonotope methods, such as those used in DeepZ, and the polytope method used in DeepPoly.


\end{abstract}

\section{Introduction}


%
%

Convolutional neural networks (CNN) have rapidly accelerated progress in computer vision with many practical applications such as face recognition \cite{lawrence1997face}, image classification \cite{krizhevsky2012imagenet}, document analysis \cite{lecun1998gradient} and semantic segmentation. Recently, it has been shown that CNNs are vulnerable to adversarial attacks, where a well-trained CNN can be fooled into producing errant predictions due to tiny changes in their inputs \cite{goodfellow2014explaining}. Many applications such as autonomous driving seek to leverage the power of CNNs. However due the opaque nature of these models there are reservations about using in safety-critical applications. Thus, there is an urgent need for formally evaluating the robustness of a trained CNN.

The formal verification of deep neural networks (DNNs) has recently become a hot topic. The majority of the existing approaches focus on verifying safety and robustness properties of feedforward neural networks (FNN) with the Rectified Linear Unit activation function (ReLU). These approaches include: mixed-integer linear programming (MILP)~\cite{lomuscio2017approach,kouvaros2018formal,dutta2017output}, satisfiability (SAT) and satisfiability modulo theory (SMT) techniques \cite{katz2017reluplex}, optimization \cite{weng2018towards,zhang2018efficient,lin2019robustness,hein2017formal,dvijotham2018dual,wu2019game}, and geometric reachability \cite{tran2019parallel,wang2018formal,xiang2018output,xiang2017reachable,xiang2018specification,singh2019abstract,singh2018fast,tran2019fm, wong2017provable, xiaodong2020facelatice}. Adjacent to these methods are property inference techniques for DNNs, which are also an important and interesting research area investigated in \cite{8952519}. In a similar fashion,  the problem of verifying the safety of cyber-physical systems (CPS) with learning-enabled neural network components with imperfect plant models and sensing information~\cite{tran2019emsoft,souradeep2019,xiang2018reachable,sun2018formal,ivanov2018verisig, xiang2020tnnls, ivanov2020case, huang2019reachnn, akintunde2020formal, manzanas2019arch} has recently attracted significant attention due to their real world applications. This research area views the safety verification problem in a more holistic manner by considering the safety of the whole system in which learning-enabled components interact with the physical world.

Although numerous tools have been proposed for neural network verification, only a handful of methods can deal with CNNs \cite{kouvaros2018formal,singh2019abstract,singh2018fast,katz2019marabou,ruan2018global,anderson2019pldi}. Moreover, in the aforementioned techniques, only one \cite{ruan2018global} can deal with real-world CNNs, such as VGGNet \cite{simonyan2014very}. Their approach makes used of the concept of the $L_0$ distance between two images. Their optimization-based approach computes a tight bound on the number of pixels that may be changed in an image without affecting the classification result of the network. It can also efficiently generate adversarial examples that can be used to improve the robustness of network. In a similar manner, this paper seeks to verify the robustness of real-world deep CNNs. Thus, we propose a set-based analysis method through the use of the \emph{ImageStar}, a new set representation that can represent an infinite number of images. As an example, this representation can be used to represent a set of images distorted by an adversarial attack. Using the ImageStar, we propose both exact and over-approximate reachability algorithms to construct reachable sets that contain all the possible outputs of a CNN under an adversarial attack. These reachable sets are then used to reason about the overall robustness of the network. When a CNN violates a robustness property, our exact reachability scheme can construct a \emph{set of concrete adversarial examples}. Our approach differs from \cite{ruan2018global} in two ways. First, our method does not provide robustness guarantees for a network in terms of the number of pixels that are allowed to be changed (in terms of $L_0$ distance). Instead, we prove the robustness of the network on images that are attacked by disturbances bounded by arbitrary linear constraints. Second, our approach relies on reachable set computation of a network corresponding to a bounded input set, as opposed to a purely optimization-based approach.

We implement the proposed method in a tool called NNV and compare it with the zonotope method used in DeepZ \cite{singh2018fast} and the polytope method used in DeepPoly \cite{singh2019abstract}.
The experimental results show that our method is less conservative and faster than any of these approaches when verifying the robustness of CNNs.
The main contributions of the paper are as follows.
\vspace{-0.75em}%
\begin{itemize}
	\item The ImageStar set representation, which is an efficient representation for reachability analysis of CNNs.
	\item The provision of exact and over-approximate reachability algorithms for constructing reachable sets and verifying robustness of CNNs.
	\item The implementation of the ImageStar representation and reachability algorithms in NNV \cite{tran2020cav_tool}.
    \item Rigorous evaluation and comparison of proposed approaches, such as zonotope \cite{singh2018fast} and polytope \cite{singh2019abstract} methods on different CNNs.
\end{itemize}

\section{Problem formulation}

The reachability problem for CNNs is the task of analyzing a trained CNN with respect to some perturbed input set in order to construct a set containing all possible outputs of the network. In this paper, we consider the reachability of a CNN $\mathcal{N}$ that consists of a series of layers $L$ that may include convolutional layers, fully connected layers, max-pooling layers, average pooling layers, and ReLU activation layers. Mathematically, we define a CNN with $n$ layers as $\mathcal{N} = \{L_i\}, i=1,2,\dots, n$. The reachability of the CNN $\mathcal{N}$ is defined based on the concept of \emph{reachable sets}.

\begin{definition}[Reachable set of a CNN]\label{def:reachable set}
An (output) reachable set $\mathcal{R}_{\mathcal{N}}$ of a CNN $\mathcal{N} = \{L_i\}, i=1,2,\dots, n$ corresponding to a linear input set $\mathcal{I}$ is defined incrementally as:%
\begin{equation*}%
    \begin{split}
      \mathcal{R}_{L_{1}} &\triangleq \{y_1~|~y_1 = L_1(x),~ x \in \mathcal{I}\}, \\
      \mathcal{R}_{L_{2}} &\triangleq \{y_2~|~y_2 = L_2(y_1),~ y_1 \in \mathcal{R}_{L_1}\}, \\
     &\vdots \\
    \mathcal{R}_{\mathcal{N}} = \mathcal{R}_{L_n} &\triangleq \{y_n ~|~ y_n = L_3(y_{n-1})\,~y_{n-1} \in \mathcal{R}_{L_{n-1}} \}, \\
    \end{split}
  \end{equation*}
where $L_i(\cdot)$ is a function representing the operation of the $i^{th}$ layer.
\end{definition}

The definition shows that the reachable set of the CNN $\mathcal{N}$ can be constructed \emph{layer-by-layer}. The core computation is constructing the reachable set of each layer $L_i$ defined by a specific operation, i.e., convolution, affine mapping, max pooling, average pooling, or ReLU.



\section{ImageStar}

\begin{definition}\label{def:star}
\textbf{An ImageStar} $\Theta$ is a tuple $\langle c, V, P \rangle$ where $c \in \mathbb{R}^{h\times w \times nc}$ is the anchor image, $V = \{v_1, v_2, \cdots, v_m\}$ is a set of m images in $\mathbb{R}^{h \times w \times nc}$ called generator images, $P: \mathbb{R}^m \to \{ \top, \bot\}$ is a predicate, and $h, w, nc$ are the height, width and number of channels of the images respectively. The generator images are arranged to form the ImageStar's $h\times w \times nc \times m$ basis array. The set of images represented by the ImageStar is given as:
\begin{equation*}
 \llbracket \Theta \rrbracket = \{x~|~x = c + \Sigma_{i=1}^m(\alpha_iv_i)~\text{such that}~P(\alpha_1, \cdots, \alpha_m) = \top \}.
\end{equation*}
Sometimes we will refer to both the tuple $\Theta$ and the set of states $\llbracket \Theta \rrbracket$ as $\Theta$. In this work, we restrict the predicates to be a conjunction of linear constraints, $P(\alpha) \triangleq C\alpha \leq d$ where, for $p$ linear constraints, $C \in \mathbb{R}^{p \times m}$, $\alpha$ is the vector of $m$-variables, i.e., $\alpha = [\alpha_1, \cdots, \alpha_m]^T$, and $d \in \mathbb{R}^{p \times 1}$. A ImageStar is an empty set if and only if $P(\alpha)$ is empty.

\end{definition}

\begin{example}[ImageStar]\label{ex1}
A $4\times 4 \times 1$ gray image with a bounded disturbance $ b \in [-2, 2]$ applied on the pixel of the position $(1,2,1)$ can be described as an ImageStar depicted in Figure \ref{fig:ImageStar-example}.
\end{example}
\begin{figure}[t]
  \centering
      \includegraphics[width=0.6\textwidth]{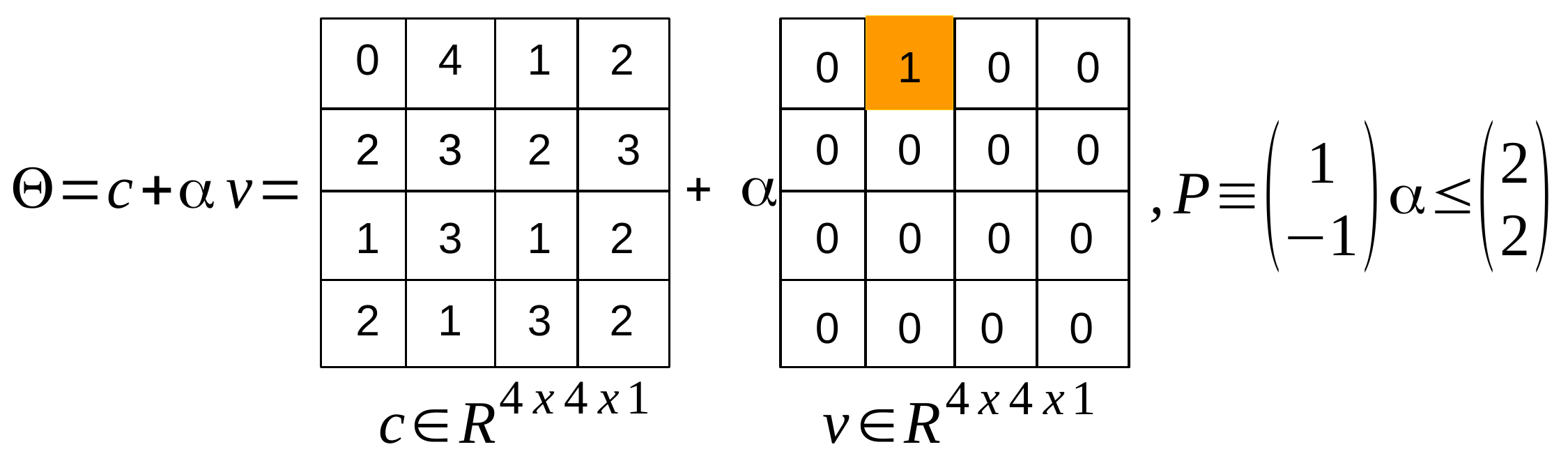}
  \caption{An example of an ImageStar.}
\vspace{-1em}
  \label{fig:ImageStar-example}
\end{figure}

\begin{remark}
An ImageStar is an extension of the generalized star set recently defined in \cite{bak2017simulation,bak2019numerical,tran2019formats,tran2019fm}. In a generalized star set, the anchor and the generators are vectors, while in an ImageStar, the anchor and generators are images with multiple channels. We will later show that the ImageStar is a very efficient representation for the reachability analysis of convolutional layers, fully connected layers, and average pooling layers.
\end{remark}

\begin{proposition}[Affine mapping of an ImageStar]  An affine mapping of an ImageStar $\Theta = \langle c, V, P \rangle $ with a scale factor $\gamma$ and an offset image $\beta$ is another ImageStar $\Theta^{\prime} = \langle c^{\prime}, V^{\prime}, P^{\prime} \rangle $ in which the new anchor, generators and predicate are as follows.
\begin{equation*}
c^{\prime} = \gamma \times c + \beta, ~~V^{\prime} = \gamma \times V, ~~ P^{\prime} \equiv P.
\end{equation*}
Note that, the scale factor $\gamma$ can be a scalar or  a vector containing scalar scale factors in which each factor is used to scale one channel in the ImageStar.
\end{proposition}


\section{Reachability of CNN using ImageStars}
In this section, we present the reachable set computation for the convolutional layer, average pooling layer, fully connected layer, batch normalization layer, max pooling layer, and the ReLU layer with respect to an input set consisting of an ImageStar. 

\subsection{Reachability of a convolutional layer}
We consider a two-dimensional convolutional layer with following parameters: the weights $W_{Conv2d} \in \mathbb{R}^{h_f \times w_f \times nc \times nf}$, the bias $b_{Conv2d} \in \mathbb{R}^{1\times 1\times nf}$, the padding size $P$, the stride $S$, and the dilation factor $D$ where $h_f, w_f, nc$ are the height, width, and the number of channels of the filters in the layer respectively. Additionally, $nf$ is the number of filters. The reachability of a convolutional layer is given in the following lemma.
\begin{lemma}\label{lm:convolutional layer}
The reachable set of a convolutional layer with an ImageStar input set $\mathcal{I} = \langle c, V, P \rangle$ is another ImageStar $\mathcal{I}^{\prime} = \langle c^{\prime}, V^{\prime}, P\rangle$ where $c^{\prime} = Convol(c)$ is the convolution operation applied to the anchor image, $V^{\prime} = \{v_1^{\prime}, \dots, v_m^{\prime} \}, v_i^{\prime} = ConvolZeroBias(v_i)$ is the convolution operation with zero bias applied to the generator images, i.e., only using the weights of the layer.
\end{lemma}
\begin{proof}
Any image in the ImageStar input set is a \emph{linear} combination of the center and basis images. For any filter in the layer, the convolution operation applied to the input image performs local element-wise multiplication of a local matrix (of all channels) containing the values of the local pixels of the image and the the weights of the filter and then combine the result with the bias to get the output for that local region. Due to the linearity of the input image, we can perform the convolution operation with the bias on the center and the convolution operation with zero bias on the basis images and then combine the result to get the output image.
\end{proof}

\begin{example}[Reachable set of a convolutional layer]
The reachable set of a convolutional layer with single $2\times 2$ filter and the ImageStar input set in Example \ref{ex1} is described in Figure \ref{fig:reach-conv2d}, where the weights and the bias of the filter are $W = \begin{bmatrix} 1 & 1 \\ -1 & 0 \\ \end{bmatrix}$,  $b = -1$ respectively, the stride is $S = [2~2]$, the padding size is $P = [0~0~0~0]$ and the dilation factor is $D = [1~1]$.
\end{example}
\begin{figure}[]
  \centering
      \includegraphics[width=0.6\textwidth]{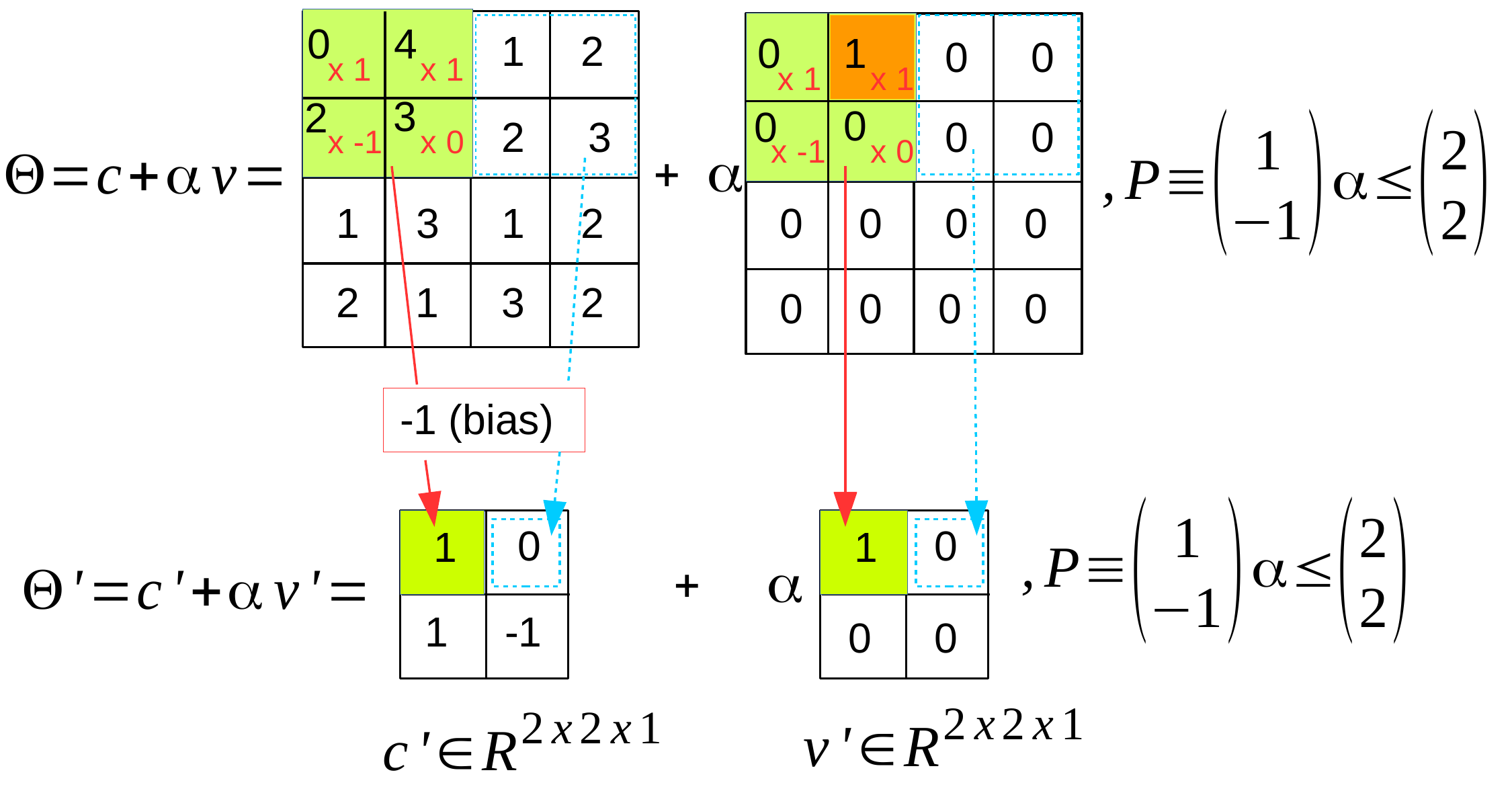}
  \caption{Reachability of convolutional layer using ImageStar.}
  \label{fig:reach-conv2d}
  \vspace{-2em}
\end{figure}

\subsection{Reachability of an average pooling layer}
The reachability of an average pooling layer with pooling size $PS$, padding size $P$, and stride $S$ is given below, with its proof similar to that of the convolutional layer.
\begin{lemma}\label{lm:average pooling}
The reachable set of a average pooling layer with an ImageStar input set $\mathcal{I} = \langle c, V, P \rangle$ is another ImageStar $\mathcal{I}^{\prime} = \langle c^{\prime}, V^{\prime}, P\rangle$ where $c^{\prime} = average(c)$, $V^{\prime} = \{v_1^{\prime}, \dots, v_m^{\prime} \}, v_i^{\prime} = average(v_i)$, $average(\cdot)$ is the average pooling operation applied to the anchor and generator images.
\end{lemma}
%
%
\begin{example}[Reachable set of an average pooling layer]
The reachable set of an $2\times 2$ average pooling layer with padding size $P = [0~0~0~0]$, stride $S = [2~2]$, and an ImageStar input set given by Example \ref{ex1} is shown in Figure \ref{fig:reach-average-pooling}.
\end{example}
\begin{figure}[]%
\vspace{-1em}%
  \centering%
      \includegraphics[width=0.6\textwidth]{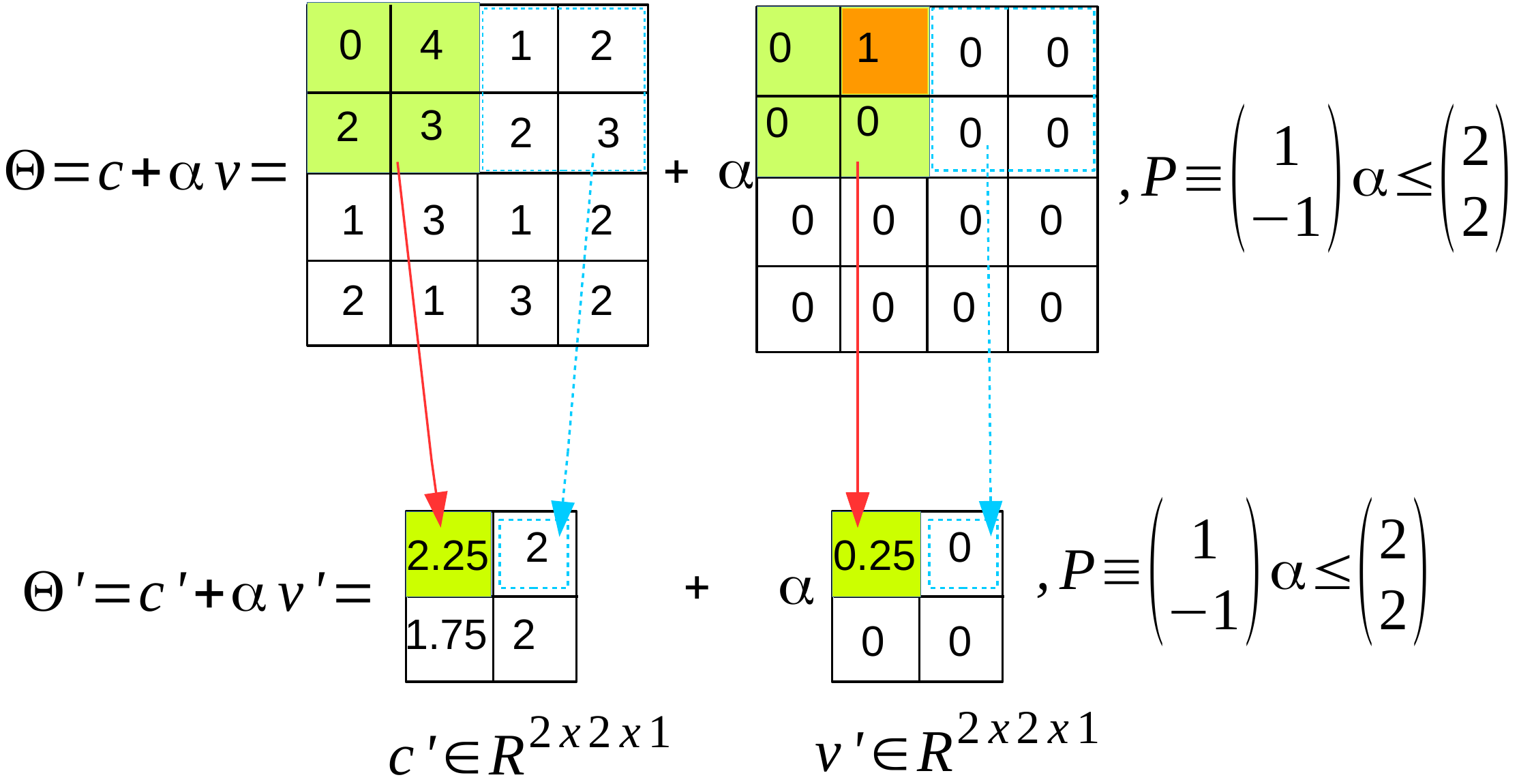}%
  \caption{Reachability of average pooling layer using ImageStar.}%
  \label{fig:reach-average-pooling}%
  \vspace{-2em}%
\end{figure}

\subsection{Reachability of a fully connected layer}
The reachability of a fully connected layer is stated in the following lemma.
\begin{lemma}\label{lm:fully-connected-layer}
 Given a two-dimensional fully connected layer with weight $W_{fc} \in \mathbb{R}^{n_{fc} \times m_{fc}}$, bias $b_{fc} \in \mathbb{R}^{n_{fc}}$, and an ImageStar input set  $\mathcal{I} = \langle c, V, P \rangle$, the reachable set of the layer is another ImageStar $\mathcal{I}^{\prime} = \langle c^{\prime}, V^{\prime}, P\rangle$ where $c^{\prime} = W*\bar{c} + b$, $V^{\prime} = \{v_1^{\prime}, \dots, v_m^{\prime} \}, v_i^{\prime} = W_{fc}*\bar{v}_i$, $\bar{c}(\bar{v_i}) = reshape(c(v_i), [m_{fc}, 1])$. Note that it is required for consistency between the ImageStar and the weight matrix that $m_{fc} = h\times w \times nc$, where $h, w, nc$ are the height, width and number of channels of the ImageStar.
\end{lemma}
\begin{proof}
Similar to the convolutional layer and  the average pooling layer, for any image in the ImageStar input set, the fully connected layer performs an affine mapping of the input image which is a linear combination of the center and the basis images of the ImageStar. Due to the linearity, the affine mapping of the input image can be decomposed into the affine mapping of the center image and the affine mapping without the bias of the basis images. The final result is the sum of the individual affine maps.
\end{proof}

\subsection{Reachability of a batch normalization layer}
%
In the prediction phase, a batch normalization layer normalizes each input channel $x_i$ using the mean $\micro$ and variance $\sigma^2$ over the full training set.
Then the batch normalization layer further shifts and scales the activations using the offset $\beta$ and the scale factor $\gamma$ that are learnable parameters.
The formula for normalization is as follows.
\begin{equation*}
\bar{x}_i = \frac{x_i - \micro}{\sqrt{\sigma^2 + \epsilon}}, ~~ y_i = \gamma \bar{x}_i + \beta.
\end{equation*}
where $\epsilon$ is a used to prevent division by zero. The batch normalization layer can be described as a tuple $\mathcal{B} = \langle \micro, \sigma^2, \epsilon, \gamma, \beta \rangle$.
The reachability of a batch normalization layer with an ImageStar input set is given in the following lemma.
\begin{lemma}
  The reachable set of a batch normalization layer $\mathcal{B} = \langle \micro, \sigma^2, \epsilon, \gamma, \beta \rangle$ with an ImageStar input set $\mathcal{I} = \langle c, V, P \rangle$ is another ImageStar $\mathcal{I}^{\prime} = \langle c^{\prime}, V^{\prime}, P^{\prime}\rangle$ where:
\begin{equation*}
c^{\prime} = \frac{\gamma}{\sqrt{\sigma^2 + \epsilon}}c + \beta - \frac{\gamma}{\sqrt{\sigma^2 + \epsilon}}\micro,~~V^{\prime} = \frac{\gamma}{\sqrt{\sigma^2 + \epsilon}}V, ~~P^{\prime} \equiv P.
\end{equation*}
\end{lemma}
\begin{proof}
The reachable set of a batch normalization layer can be obtained in a straightforward fashion using two affine mappings of the ImageStar input set.
\end{proof}

\subsection{Reachability of a max pooling layer}
Reachability of max pooling layer with an ImageStar input set is challenging because the value of each pixel in an image in the ImageStar depends on the predicate variables $\alpha_i$. Therefore, the local max point when applying max-pooling operation may change with the values of the predicate variables. In this section, we investigate the exact reachability and over-approximate reachability of a max pooling layer with an ImageStar input set. The first obtains the exact reachable set while the second constructs an over-approximate reachable set.

\subsubsection{Exact reachability of a max pooling layer}
The central idea in the exact analysis of the max-pooling layer is finding a set of \emph{local max point candidates} when we apply the max pooling operation on the image. We consider the max pooling operation on the ImageStar in Example $\ref{ex1}$ with a pool size of $2 \times 2$, a padding size of $P = [0~0~0~0]$, and a stride $S = [2~2]$ to clarify the exact analysis step-by-step. First, the max-pooling operation is applied on $4$ local regions $I, II, III, IV$, as shown in Figure \ref{fig:reach-exact-max-pooling}. The local regions $II, III, IV$ have only one \emph{max point candidate} whic is the pixel that has the maximum value in the region. It is interesting to note that region $I$ has two max point candidates at the positions $(1, 2, 1)$ and $(2, 2, 1)$ and these candidates correspond to different conditions of the predicate variable $\alpha$. For example, the pixel at the position $(1, 2, 1)$ is the max point if and only if $4 + \alpha \times 1 \geq 3 + \alpha \times 0 $. Note that with $-2 \leq \alpha \leq 2$, we always have $4 + \alpha * 1 \geq 2 + \alpha \times 0 \geq 0 + \alpha \times 0$. Since the local region $I$ has two max point candidates, and other regions have only one, the exact reachable set of the max-pooling layer is the union of two new ImageStars $\Theta_1$ and $\Theta_2$. In the first reachable set $\Theta_1$, the max point of the region $I$ is $(1, 2, 1)$ with an additional constraint on the predicate variable $\alpha \geq -1$. For the second reachable set $\Theta_2$, the max point of the region $I$ is $(2, 2, 1)$ with an additional constraint on the predicate variable $\alpha \leq -1$. One can see that from a single ImageStar input set, the output reachable set of the max-pooling layer is split into two new ImageStars. Therefore, the number of ImageStars in the reachable set of the max-pooling layer may grow quickly if each local region has more than one max point candidates. The worst-case complexity of the number of ImageStars in the exact reachable set of the max-pooling layer is given in Lemma \ref{lm:complexity-max-pooling}. The exact reachability algorithm is presented in the Appendix \ref{sec:exact-maxpool}.
\begin{figure}[H]
  \centering
      \includegraphics[width=0.6\textwidth]{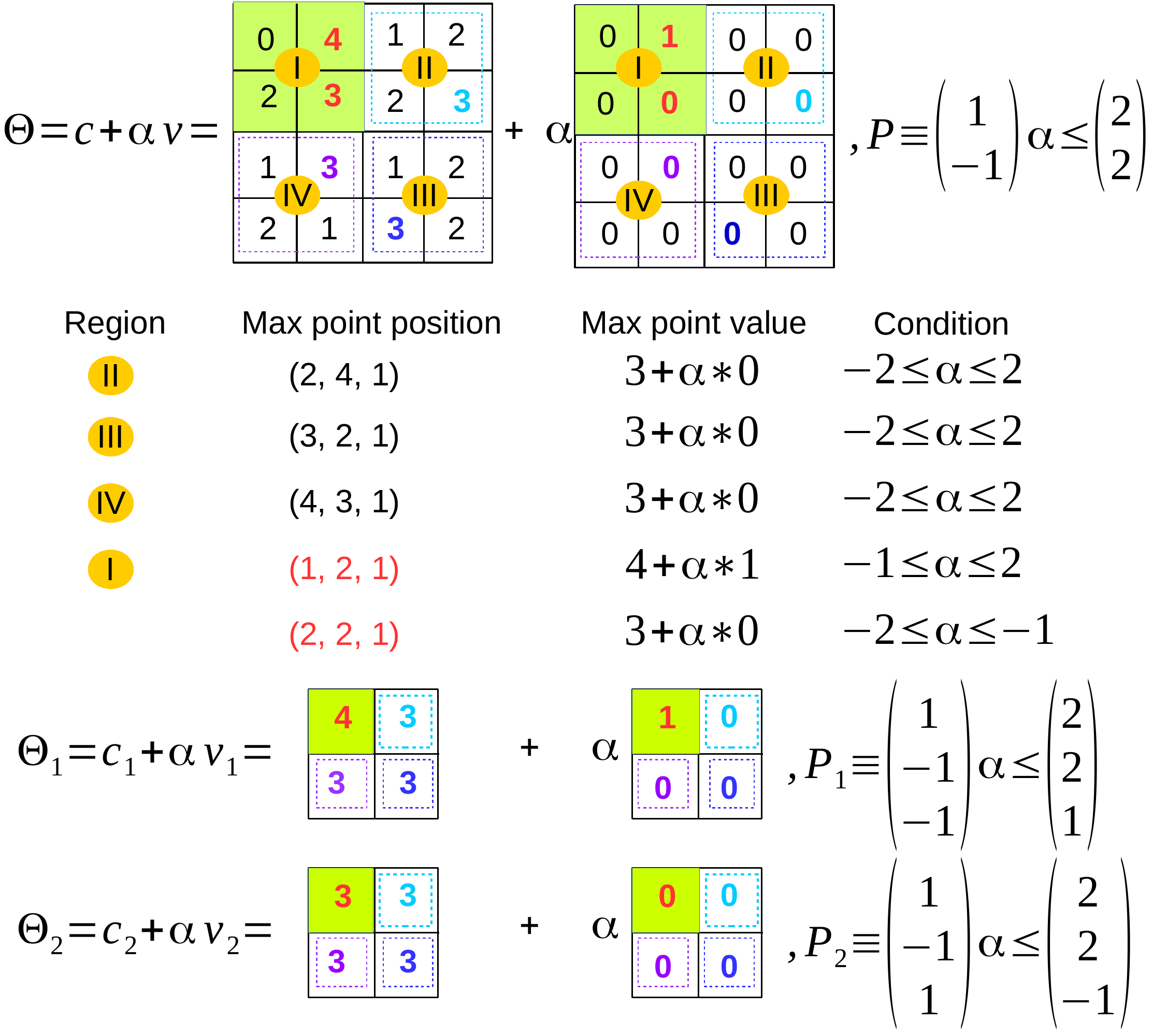}
  \caption{Exact reachability of max pooling layer using ImageStars.}
  \label{fig:reach-exact-max-pooling}
  \vspace{-1.5em}
\end{figure}
\begin{lemma}\label{lm:complexity-max-pooling}
The worst-case complexity of the number of ImageStars in the exact reachability of the max pooling layer is $\mathcal{O}(((p_1\times p_2)^{h \times w})^{nc})$ where $[h, w, nc]$ is the size of the ImageStar output sets, and $[p_1, p_2]$ is the size of the max-pooling layer.
\end{lemma}
\begin{proof}
An image in the ImageStar output set has $h \times w$ pixels in each channel. For each pixel, in the worst case, there are $p_1 \times p_2$ candidates. Therefore, the number of ImageStars in the output set in the worst case is $\mathcal{O}(((p_1\times p_2)^{h \times w})^{nc})$.
\end{proof}

Finding a set of local max point candidates is the core computation in the exact reachability of max-pooling layer. To optimize this computation, we divide the search for the local max point candidates into two steps. The first one is to estimate the ranges of all pixels in the ImageStar input set. We can solve $h_I \times w_I \times nc$ linear programming optimizations to find the exact ranges of these pixels, where $[h_I, w_I, nc]$ is the size of the input set. However, unfortunately this is a time-consuming computation. For example, \textbf{\textit{if a single linear optimization can be done in $0.01$ seconds, for an ImageStar of the size $224 \times 224 \times 32$, we need about $10$ hours to find the ranges of all pixels}}. To overcome this bottleneck, we quickly estimate the ranges using only the ranges of the predicate variables to get rid of a vast amount of non-max-point candidates. In the second step, we solve a much smaller number of LP optimizations to determine the exact set of the local max point candidates and then construct the ImageStar output set based on these candidates.

Lemma \ref{lm:complexity-max-pooling} shows that the number of ImageStars in the exact reachability analysis of a max-pooling layer may grow exponentially. To overcome this problem, we propose the following over-approximate reachability method.

\subsubsection{Over-approximate reachability of a max pooling layer}
The central idea of the over-approximate analysis of the max-pooling layer is that if a local region has more than one max point candidates, we introduce a \emph{new predicate variable} standing for the max point of that region. We revisit the example introduced earlier in the exact analysis to clarify this idea. Since the first local region $I$ has two max point candidates, we introduce new predicate variable $\beta$ to represent the max point of this region by adding three new constraints: 1) $\beta \geq 4 + \alpha * 1$, i.e., $\beta$ must be equal or larger than the value of the first candidate ; 2) $\beta \geq 3 + \alpha * 0$, i.e., $\beta$ must be equal or larger than the value of the second candidate; 3) $\beta \leq 6$, i.e., $\beta$ must be equal or smaller than the upper bound of the pixels values in the region. With the new predicate variable, a single over-approximate reachable set $\Theta^{\prime}$ can be constructed in Figure \ref{fig:reach-approx-max-pooling}. The approximate reachability algorithm is presented in the Appendix \ref{sec:approx-maxpool}.
\begin{figure}[]
  \centering
      \includegraphics[width=0.6\textwidth]{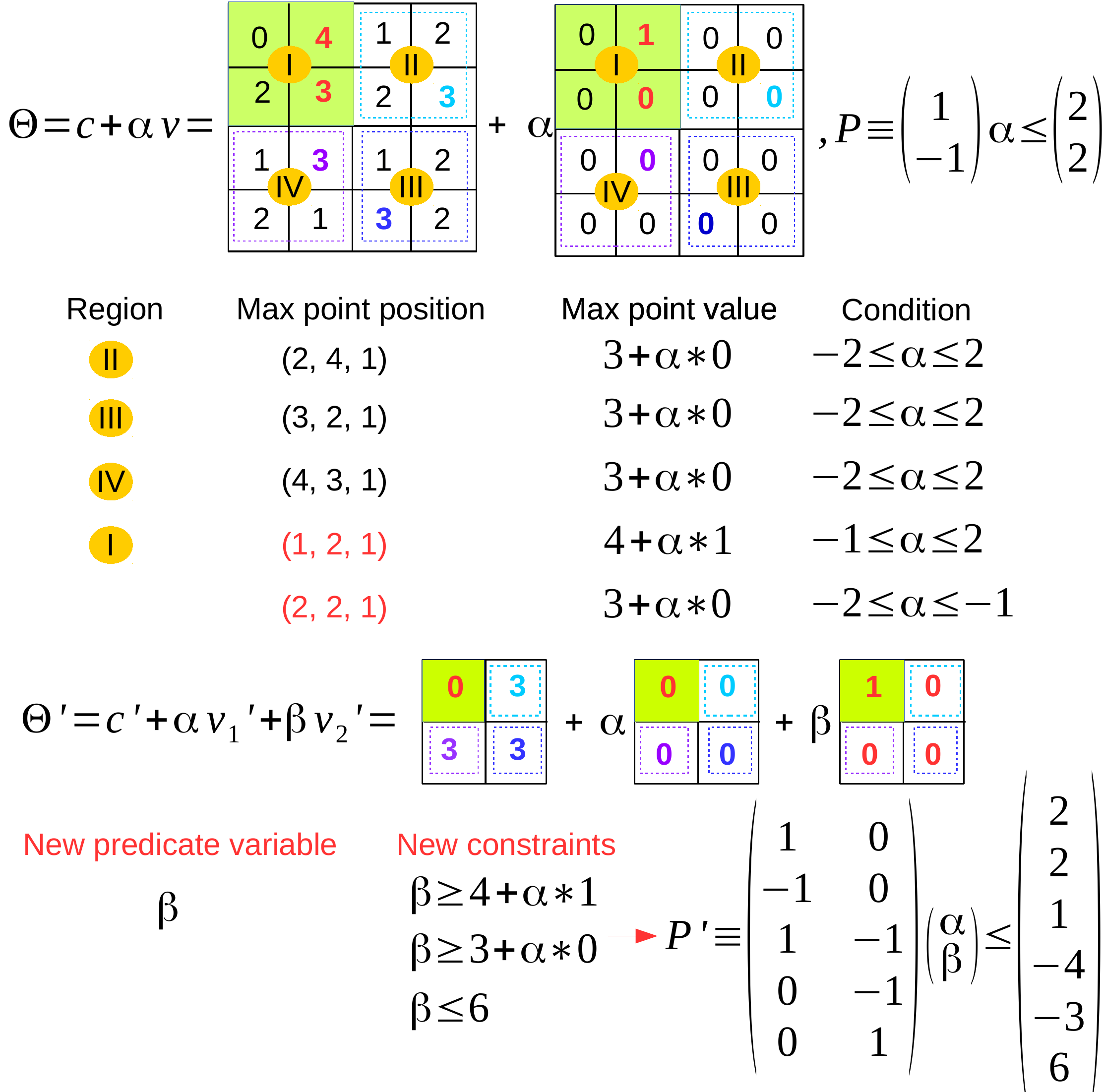}
  \caption{Over-approximate reachability of max pooling layer using ImageStar.}
  \label{fig:reach-approx-max-pooling}
\vspace{-2em}
\end{figure}
\begin{lemma}\label{lm:complexity-approx-maxpooling-layer}
The worst-case complexity of the new predicate variables introduced in the over-approximate analysis is $\mathcal{O}(h\times w \times nc)$ where $[h, w, nc]$ is the size of the ImageStar output set.
\end{lemma}

\subsection{Reachability of a ReLU layer}
Similar to max-pooling layer, the reachability analysis of a ReLU layer is also challenging because the value of each pixel in an ImageStar may be smaller than zero or larger than zero depending on the values of the predicate variables ($ReLU(x) = max(0, x)$). In this section, we investigate the exact and over-approximate reachability algorithms for a ReLU layer with an ImageStar input set. The techniques we use in this section are adapted from in \cite{tran2019fm}.

\subsubsection{Exact reachability of a ReLU layer}
The central idea of the exact analysis of a ReLU layer with an ImageStar input set is performing a sequence of \emph{stepReLU operations} over all pixels of the ImageStar input set. Mathematically, the exact reachable set of a ReLU layer $L$ can be computed as follows.
\begin{equation*}
\mathcal{R}_L = stepReLU_N(stepReLU_{N-1}(\dots (stepReLU_1(\mathcal{I})))),
\end{equation*}
where $N$ is the total number of pixels in the ImageStar input set $\mathcal{I}$.
The $stepReLU_i$ operation determines whether or not a split occurs at the $i^{th}$ pixel. If the pixel value is larger than zero, then the output value of that pixel remains the same. If the pixel value is smaller than zero than the output value of that pixel is reset to be zero. The challenge is that the pixel value depends on the predicate variables. Therefore, there is the case that the pixel value may be negative or positive with \emph{an extra condition} on the predicate variables. In this case, we split the input set into two \emph{intermediate} ImageStar reachable sets and apply the ReLU law on each intermediate reach set. An example of the stepReLU operation on an ImageStar is illustrated in Figure \ref{fig:reach-exact-relu}. The value of the first pixel value $-1 + \alpha$ would be larger than zero if $\alpha \leq 1$, and in this case we have $ReLU(-1 + \alpha) = -1 + \alpha$. If $\alpha <= 1$, then $ReLU(-1 + \alpha) = 0 + \alpha \times 0$. Therefore, the first stepReLU operation produces two intermediate reachable sets $\Theta_1$ and $\Theta_2$, as shown in the figure. The number of ImageStars in the exact reachable set of a ReLU layer increases quickly along with the number of splits in the analysis, as stated in the following lemma.
\begin{figure}[]
\vspace{-1em}
  \centering
      \includegraphics[width=0.5\textwidth]{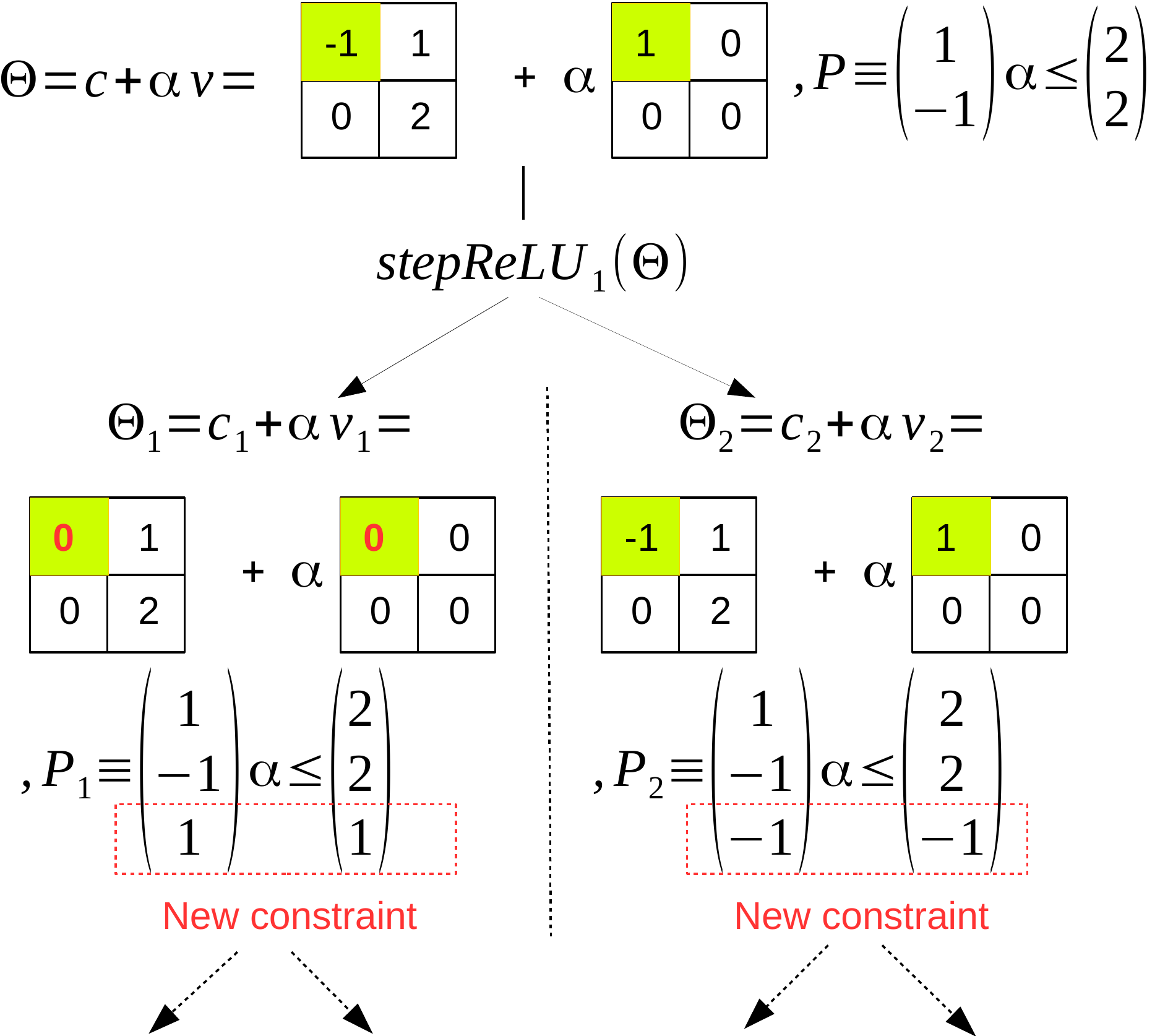}
  \caption{stepReLU operation on an ImageStar.}
  \label{fig:reach-exact-relu}
  \vspace{-2em}
\end{figure}

\begin{lemma}
The worst-case complexity of the number of ImageStars in the exact analysis of a ReLU layer is $\mathcal{O}(2^N)$, where $N$ is the number of pixels in the ImageStar input set.
\end{lemma}
\begin{proof}
There are $h \times w \times nc$ local regions in the approximate analysis. In the worst case, we need to introduce a new variable for each region. Therefore, the worst case complexity of new predicate variables introduced is $\mathcal{O}(h \times w \times nc)$.
\end{proof}

Similar to \cite{tran2019fm}, to control the explosion in the number of ImageStars in the exact reachable set of a ReLU layer, we propose an over-approximate reachability algorithm in the following.

\subsubsection{Over-approximate reachability of a ReLU layer}
The idea behind the over-approximate reachability of ReLU layer is replacing the stepReLU operation at each pixel in the ImageStar input set by an \emph{approxStepReLU} operation. At each pixel where a split occurs, we introduce a new predicate variable to over-approximate the result of the stepReLU operation at that pixel. An example of the overStepReLU operation on an ImageStar is depicted in Figure \ref{fig:reach-approx-relu} in which the first pixel of the input set has the ranges of $[l_1 = -3, u_1 = 1]$ indicating that a split occurs at this pixel. To avoid this split, we introduce a new predicate variable $\beta$ to over-approximate the exact intermediate reachable set (i.e., two blue segments in the figure) by a triangle. This triangle is determined by three constraints: 1) $\beta \geq 0$ (the $ReLU(x) \geq 0$ for any $x$); 2) $\beta \geq -1 + \alpha$ ($ReLU(x) \geq x$ for any $x$); 3) $\beta \leq 0.5 + 0.25 \alpha$ (upper bound of the new predicate variable). Using this over-approximation, a single intermediate reachable set $\Theta^{\prime}$ is produced as shown in the figure. After performing a sequence of approxStepReLU operations, we obtain a single over-approximate ImageStar reachable set for the ReLU layer. However, the number of predicate variables and the number of constraints in the obtained reachable set increase.
\begin{figure}[]
  \centering
      \includegraphics[width=0.6\textwidth]{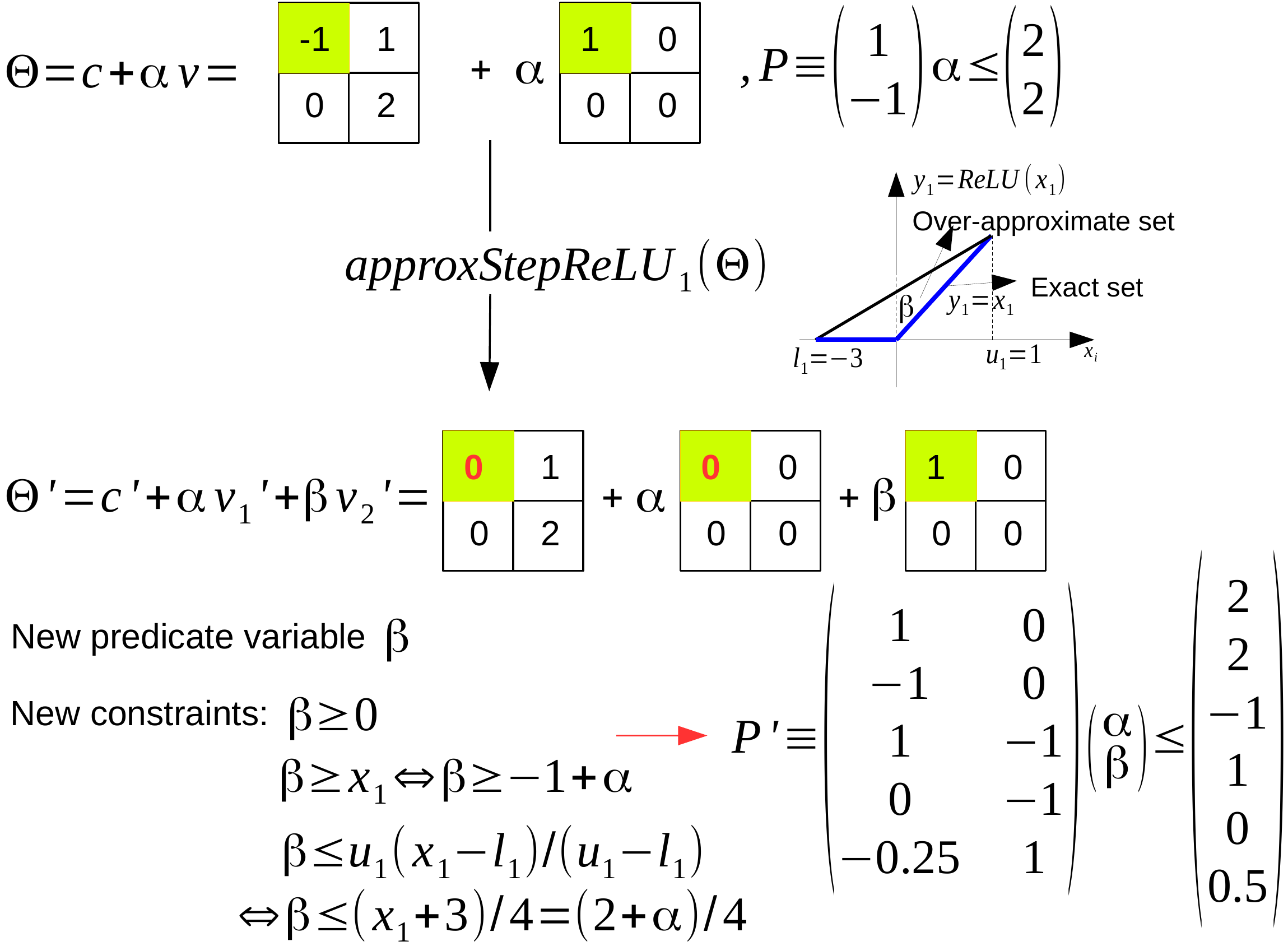}
  \caption{approxStepReLU operation on an ImageStar.}
  \label{fig:reach-approx-relu}
  \vspace{-1.5em}
\end{figure}

\begin{lemma}\label{lm:complexity-relu-predicate}
The worst case complexity of the increment of predicate variables and constraints is $\mathcal{O}(N)$ and $\mathcal{O}(3\times N)$ respectively, where $N$ is the number of pixels in the ImageStar input set.
\end{lemma}
\begin{proof}
In the worst case, splits occur at all $N$ pixels in the ImageStar input set. In this case, we need to introduce $N$ new predicate variables to over-approximate the exact intermediate reachable set. For each new predicate variable, we add $3$ new constraints.
\end{proof}

One can see that determining where splits occur is crucial in the exact and over-approximate analysis of a ReLU layer. To do this, we need to know the ranges of all pixels in the ImageStar input set. However, as mentioned earlier, the computation of the exact range is expensive. To reduce the computation cost, we first use the estimated ranges of all pixels to get rid of a vast amount of non-splitting pixels. Then we compute the exact ranges for the pixels where splits may occur to compute the exact or over-approximate reachable set of the layer.

\subsection{Reachabilty algorithm and parallelization}
We have presented the core ideas for reachability analysis of different types of layers in a CNN. The reachable set of a CNN is constructed layer-by-layer in which the output reachable set of the previous layer is the input for the next layer. For the convolutional layer, average pooling layer and fully connected layer, we always can compute efficiently the exact reachable set of each layer. For the max pooling layer and ReLU layer, we can compute both the exact and the over-approximate reachable sets. However, the number of ImageStars in the exact reachable set may grow quickly. Therefore, \textbf{\textit{in the exact analysis,  a layer may receive multiple input sets which can be handled in parallel to speed up the computation time}}. The reachability algorithm for a CNN is summarized in Algorithm \ref{alg:reach-CNN}. The detail implementation of the reachability algorithm for each layer can be found in NNV \cite{tran2020cav_tool}.
\begin{algorithm}[H]
	\caption{Reachability analysis for a CNN.}\label{alg:reach-CNN}
	\begin{algorithmic}[1]
        \footnotesize
		\Input{$\mathcal{N} = \{L_i\}_1^n$, $\mathcal{I}$, $scheme$ (\texttt{'exact' or 'approx'})}

		\Output{$R_{\mathcal{N}}$}

        \Procedure{\texttt{ $R_{\mathcal{N}}$ = reach}}{$\mathcal{N}, \mathcal{I}, scheme$}
        \State $In = \mathcal{I}$
		\ParFor {\texttt{$i=1:n$}} $In = L_i.reach(In, scheme)$
                \EndParFor
	\State $R_{\mathcal{N}} = In$

	\EndProcedure
	\end{algorithmic}
\end{algorithm}

\section{Evaluation}
The proposed reachability algorithms are implemented in NNV~\cite{tran2020cav_tool}, a tool for verification of deep neural networks and learning-enabled autonomous CPS.
NNV utilizes core functions in MatConvNet \cite{vedaldi2015matconvnet} for the analysis of the convolutional and average pooling layers.
%
%
The evaluation of our approach consists of two parts.
First, we evaluate our approach in comparison with the zonotope \cite{singh2018fast} and polytope methods \cite{singh2019abstract} re-implemented in NNV via robustness verification of deep neural networks.
Second, we evaluate the scalability of our approach and the DeepPoly polytope method using real-world image classifiers, VGG16, and VGG19 \cite{simonyan2014very}.
The experiments are done on a computer with following configurations: Intel Core i7-6700 CPU @ 3.4GHz $\times$ 8 Processor, 62.8 GiB Memory, Ubuntu 18.04.1 LTS OS.\footnote{Codes are available online at \url{https://github.com/verivital/nnv/tree/master/code/nnv/examples/Submission/CAV2020_ImageStar}.}
Finally, we present the comparison with ERAN-DeepZ method on their $ConvMaxPool$ network trained on CIFAR-10 data set in the Appendix of this paper. 

\subsection{Robustness Verification of MNIST Classification Networks}
We compare our approach with the zonotope and polytope methods in two aspects including verification time and conservativeness of the results.
To do that, we train 3 CNNs a small, a medium, and a large CNN with $98\%, 99.7\%$ and $99.9\%$ accuracy respectively using the MNIST data set consisting of $60000$ images of handwritten digits with a resolution of $28 \times 28$ pixels \cite{lecun1998mnist}.
The network architectures are given in Figure \ref{fig:MNIST_networks} in the Appendix.
The networks classify images into ten classes: $0, 1, \dots, 9$.
The classified output is the index of the dimension that has maximum value, i.e., the argmax across the $10$ outputs.
We evaluate the robustness of the network under the well-known brightening attack used in \cite{gehr2018ai2}.
The idea of a brightening attack is that we can change the value of some pixels independently in the image to make it brighter or darker to fool the network, to misclassify the image.
In this case study, we darken a pixel of an image if its value $x_i$ (between 0 and 255) is larger than a threshold $d$, i.e., $x_i \geq d$.
Mathematically, we reduce the value of that pixel $x_i$ to the new value $x_i^{\prime}$ such that $0 \leq x_i^{\prime} \leq \delta \times x_i$.

The robustness verification is done as follows.
We select $100$ images that are correctly classified by the networks and perform the brightening attack on these, which are then used to evaluate the robustness of the networks.
A network is robust to an input set if, for any \emph{attacked} image, this is correctly classified by the network.
We note that the input set contains an infinite number of images.
Therefore, to prove the robustness of the network to the input set, we first compute the output set containing all possible output vectors of the network using reachability analysis.
Then, we prove that in the output set, the correctly classified output always has the maximum value compared with other outputs.
Note that we can neglect the \emph{softmax} and \emph{classoutput} layers of the networks in the analysis since we only need to know the maximum output in the output set of the last fully connected layer in the networks to prove the robustness of the network.

We are interested in the percentage of the number of input sets that a network is provably robust and the verification times of different approaches under different values of $d$ and $\theta$.
When $d$ is small, the number of pixels in the image that are attacked is large and vice versa.
For example, the average number of pixels attacked (computed on $100$ cases) corresponding to $d = 250$, $245$ and $240$ are $15$, $21$ and $25$ respectively.
The value of $\delta$ dictates the size of the input set that can be created by a specific attack. Stated differently it dictates the range in which the value of a pixel can be changed.
For example, if $d = 250$ and $\delta = 0.01$, the value of an attacked pixel many range from $0$ to $2.55$.

The experiments show that using the zonotope method, we cannot prove the robustness of any network.
The reason is that the zonotope method obtains very conservative reachable sets.
Figure \ref{fig:reachSet-Compare} illustrates the ranges of the outputs computed by our ImageStar (approximate scheme), the zonotope and polytope approaches when we attack a digit $0$ image with brightening attack in which $d = 250$ and $\delta = 0.05$.
One can see that, using ImageStar and polytope method, we can prove that the output corresponding to the digit $0$ is the one that has a maximum value, which means that the network is robust in this case.
However, the zonotope method produces very large output ranges that cannot be used to prove the robustness of the network.
The figure also shows that our ImageStar method produces tighter ranges than the polytope method, which means our result is less conservative than the one obtained by the polytope method.
We note that the zonotope method is very time-consuming.
It needs $93$ seconds to compute the reachable set of the network in this case, while the polytope method only needs $0.3$ seconds, and our approximate ImageStar method needs $0.74$ seconds.
The main reason is that the zonotope method introduces many new variables when constructing the reachable set of the network, which results in the increase in both computation time and conservativeness.

\begin{figure}[h!]%
\vspace{-1em}%
  \centering%
      \includegraphics[width=\textwidth]{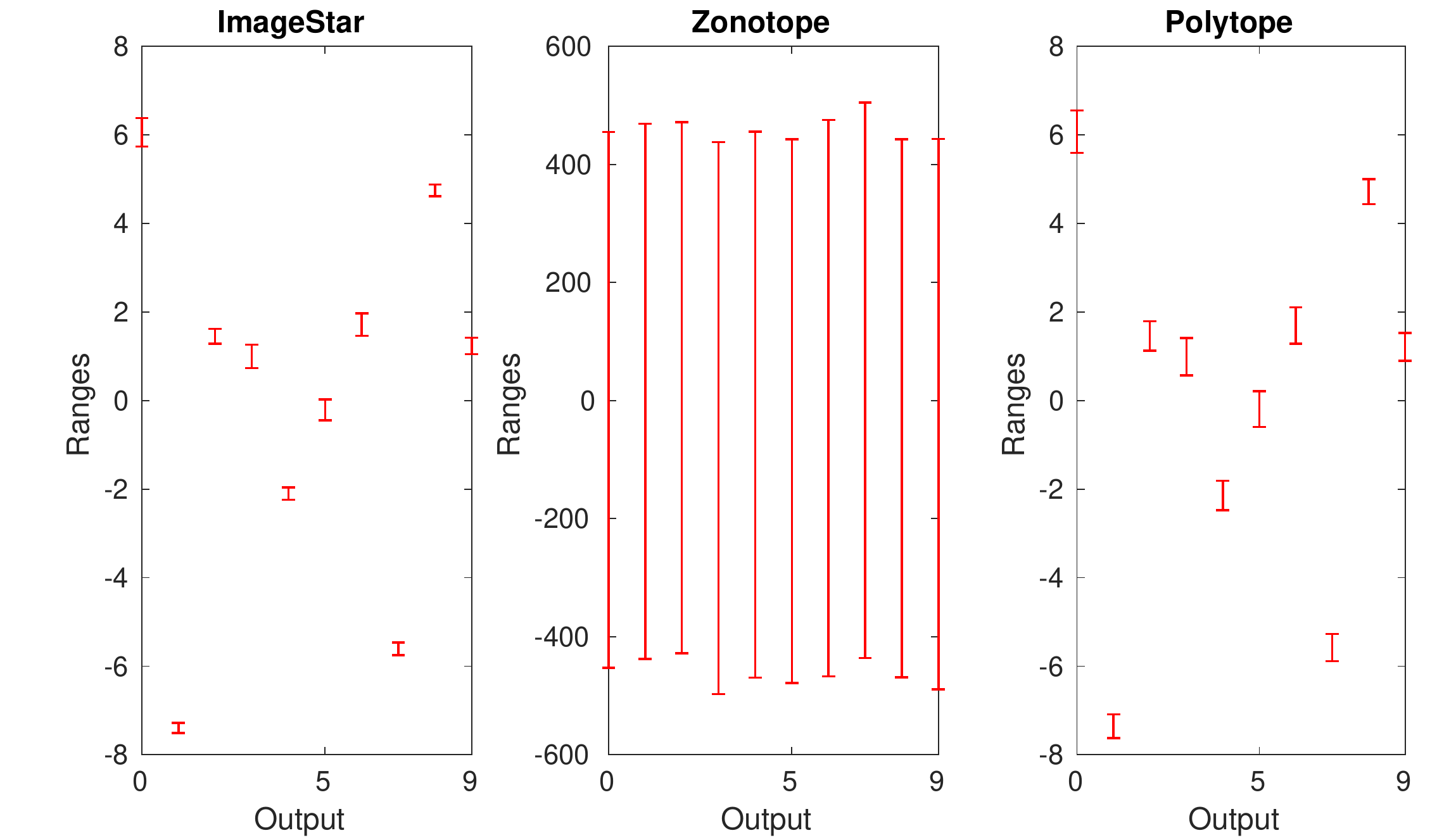}%
			\vspace{-10pt}%
  \caption{An example of output ranges of the small MNIST classification networks using different approaches.}
  \label{fig:reachSet-Compare}%
  \vspace{-2em}%
\end{figure}

The comparison of the polytope and our ImageStar method is given in Tables \ref{tab:small_convnet}, \ref{tab:medium_convnet}, and \ref{tab:large_convnet}.
The tables show that in all networks, our method is less conservative than the polytope approach since the number of cases that our approach can prove the robustness of the network is larger than the one proved by the polytope method.
For example, for the small network, for $d = 240$ and $\delta = 0.015$, we can prove $71$ cases while the polytope method can prove $65$ cases.
Importantly, the number of cases proved by DeepPoly reduces quickly when the network becomes larger.
For example, for the case that $d = 240$ and $\delta = 0.015$, the polytope method is  able to prove the robustness of the medium network for $38$ cases while our approach can prove $88$ cases.
This is because the polytope method becomes more and more conservative when the network or the input set is large.
The tables show that the polytope method is faster than our ImageStar method on the small network.
However, it is slower than the ImageStar method on any larger networks in all cases.
Notably, for the large network, the ImageStar approach is significantly faster than the polytope approach, $16.65$ times faster in average.
The results also show that the polytope approach may run into memory problem for some large input sets.

\vspace{-1em}%
\begin{table}[h]
   \centering
   \resizebox{0.67\linewidth}{!}{
\setlength{\arrayrulewidth}{.1em}
\begin{tabular}{lcccccc}
        \cline{2-7}
        & \multicolumn{6}{c}{\textbf{Robustness Results (in Percent)}}                                          \\ \cline{2-7}
        & \multicolumn{2}{c}{$\delta = 0.005$} & \multicolumn{2}{c}{$\delta = 0.01$} & \multicolumn{2}{c}{$\delta = 0.015$} \\  \cline{2-7}
        & \textit{Polytope} & \textit{ImageStar~~}       & \textit{~~Polytope}       & \textit{ImageStar~~}       & \textit{~~Polytope}        & \textit{ImageStar}       \\ \cline{1-7}
$d = 250$& $86.00$ & $87.00$& $84.00$ & $87.00$& $83.00$ & $87.00$\\
$d = 245$& $77.00$ & $78.00$& $72.00$ & $78.00$& $70.00$ & $77.00$\\
$d = 240$& $72.00$ & $73.00$& $67.00$ & $72.00$& $65.00$ & $71.00$\\ \cline{1-7}
        & \multicolumn{6}{c}{\textbf{Verification Times (in Seconds)}}                                             \\ \cline{1-7}

$d = 250$& $11.24$ & $16.28$& $18.26$ & $28.19$& $26.42$ & $53.43$\\
$d = 245$& $14.84$ & $19.44$& $24.96$ & $40.76$& $38.94$ & $85.97$\\
$d = 240$& $18.29$ & $25.77$& $33.59$ & $64.10$& $54.23$ & $118.58$\\ \hline
\end{tabular}}
\caption{Verification results of the small MNIST CNN.}
\label{tab:small_convnet}
\vspace{-3em}
\end{table}
\vspace{-1em}
\begin{table}[h]
   \centering
   \resizebox{0.67\linewidth}{!}{
\setlength{\arrayrulewidth}{.1em}
\begin{tabular}{lcccccc}
        \cline{2-7}
        & \multicolumn{6}{c}{\textbf{Robustness Results (in Percent)}}                                          \\ \cline{2-7}
        & \multicolumn{2}{c}{$\delta = 0.005$} & \multicolumn{2}{c}{$\delta = 0.01$} & \multicolumn{2}{c}{$\delta = 0.015$} \\  \cline{2-7}
        & \textit{Polytope} & \textit{ImageStar~~}       & \textit{~~Polytope}       & \textit{ImageStar~~}       & \textit{~~Polytope}        & \textit{ImageStar}       \\ \cline{1-7}
$d = 250$& $86.00$ & $99.00$& $73.00$ & $99.00$& $65.00$ & $99.00$\\
$d = 245$& $74.00$ & $95.00$& $58.00$ & $95.00$& $46.00$ & $95.00$\\
$d = 240$& $69.00$ & $90.00$& $49.00$ & $89.00$& $38.00$ & $88.00$\\ \cline{1-7}
        & \multicolumn{6}{c}{\textbf{Verification Times (in Seconds)}}                                             \\ \cline{1-7}

$d = 250$& $213.86$ & $52.09$& $627.14$ & $257.12$& $1215.86$ & $749.41$\\
$d = 245$& $232.81$ & $68.98$& $931.28$ & $295.54$& $2061.98$ & $1168.31$\\
$d = 240$& $301.58$ & $102.61$& $1451.39$ & $705.03$& $3148.16$ & $2461.89$\\ \hline
\end{tabular}}
\caption{Verification results of the medium MNIST CNN.}
\label{tab:medium_convnet}
\vspace{-2em}
\end{table}
\vspace{-1em}
\begin{table}[h]
   \centering
   \resizebox{0.67\linewidth}{!}{
\setlength{\arrayrulewidth}{.1em}
\begin{tabular}{lcccccc}
        \cline{2-7}
        & \multicolumn{6}{c}{\textbf{Robustness Results (in Percent)}}                                          \\ \cline{2-7}
        & \multicolumn{2}{c}{$\delta = 0.005$} & \multicolumn{2}{c}{$\delta = 0.01$} & \multicolumn{2}{c}{$\delta = 0.015$} \\  \cline{2-7}
        & \textit{Polytope} & \textit{ImageStar~~}       & \textit{~~Polytope}       & \textit{ImageStar~~}       & \textit{~~Polytope}        & \textit{ImageStar}       \\ \cline{1-7}
$d = 250$& $90.00$ & $99.00$& $83.00$ & $99.00$& $MemErr$ & $99.00$\\
$d = 245$& $91.00$ & $100.00$& $75.00$ & $100.00$& $MemErr$ & $100.00$\\
$d = 240$& $81.00$ & $99.00$& $MemErr$ & $99.00$& $MemErr$ & $99.00$\\ \cline{1-7}
        & \multicolumn{6}{c}{\textbf{Verification Times (in Seconds)}}                                             \\ \cline{1-7}

$d = 250$& $917.23$ & $67.45$& $5221.39$ & $231.67$& $MemErr$ & $488.69$\\
$d = 245$& $1420.58$ & $104.71$& $6491.00$ & $353.02$& $MemErr$ & $1052.87$\\
$d = 240$& $1872.16$ & $123.37$& $MemErr$ & $476.67$& $MemErr$ & $1522.50$\\ \hline
\end{tabular}}
\caption{Verification results of the large MNIST CNN.}
\label{tab:large_convnet}
\vspace{-2.5em}
\end{table}
\vspace{-0.25em}

\subsection{Robustness Verification of VGG16 and VGG19}
In this section, we evaluate the polytope and ImageStar methods on real-world CNNs, the VGG16 and VGG19 classification networks~\cite{simonyan2014very}.
We use Foolbox \cite{rauber2017foolbox} to generate the well-known DeepFool adversarial attacks \cite{moosavi2016deepfool} on a set of $20$ bell pepper images.
From an original image $ori\_im$, Foolbox generates an adversarial image $adv\_im$ that can fool the network.
The difference between two images is defined by $diff\_im = adv\_im - ori\_im$.
We want to verify if we apply $(l + \delta)$ percent of the attack on the original image, whether or not the network classifies the disturbed images correctly.
The set of disturbed images can be represented as an ImageStar as follows $disb\_im = ori\_im + (l + \delta)\times diff\_im$, where $l$ is the percentage of the attack at which we want to verify the robustness of the network, and $\delta$ is a small perturbation around $l$, i.e., $0 \leq \delta \leq \delta_{max}$.
Intuitively, $l$ describes how close we are to the attack, and the perturbation $\delta$ represents the size of the input set.

Table \ref{tab:VGG} shows the verification results of VGG16 and VGG19 with different levels of the DeepFool attack.
The networks are robust if they classify correctly the set of disturbed images $disb\_im$ as bell peppers.
To guarantee the robustness of the networks, the output corresponding to the bell pepper label (index $946$) needs to be the maximum output compared with others.
The table shows that with a small input set, small $\delta$, the polytope and ImageStar can prove the robustness of VGG16 and VGG19 with a reasonable amount of time.
Notably, the verification times as well as the robustness results of the polytope and ImageStar methods are similar when they deal with small input sets except for two cases where ImageStar is faster than the polytope method.
It is interesting to note that according to the verification results for the VGG and MNIST networks, deep networks seem to be more robust than shallow networks.

\begin{table}[h]
   \centering
   \resizebox{0.9\linewidth}{!}{
\setlength{\arrayrulewidth}{.1em}
\begin{tabular}{lcccc|cccc}
        \cline{2-9}
        & \multicolumn{8}{c}{\textbf{Robustness Results (in percentage)}}                                                                                                     \\ \cline{2-9}
        & \multicolumn{4}{c}{\textbf{VGG16}}                                                           & \multicolumn{4}{c}{\textbf{VGG19}}                                   \\ \cline{2-9}
        & \multicolumn{2}{c}{$\delta = 10^{-7}$}& \multicolumn{2}{c}{$\delta = 2\times 10^{-7}$} & \multicolumn{2}{c}{$\delta = 10^{-7}$} & \multicolumn{2}{c}{$\delta = 2\times 10^{-7}$} \\ \cline{2-9}
        & $Polytope$                        & $ImageStar~~$ & $~~Polytope$                         & $ImageStar$ & $Polytope$        & $ImageStar~~$       & $~~Polytope$       & $ImageStar$       \\ \cline{1-9}
$l = 0.96$& $85.00$ & $85.00$& $85.00$ & $85.00$           &                $100.00$ & $100.00$& $100.00$ & $100.00$                 \\
$l = 0.97$& $85.00$ & $85.00$& $85.00$ & $85.00$           &                $100.00$ & $100.00$& $100.00$ & $100.00$                \\
$l = 0.98$& $85.00$ & $85.00$& $85.00$ & $85.00$          &                 $95.00$ & $95.00$& $95.00$ & $95.00$                 \\ \cline{1-9}
\multicolumn{9}{c}{\textbf{Verification Times (in Seconds)}}                                                                                                                  \\  \cline{1-9}
$l = 0.96$& $319.04$ & $318.60$& $327.61$ & $319.93$          &             $320.91$ & $314.14$& $885.07$ & $339.30$                \\
$l = 0.97$& $324.93$ & $323.41$& $317.27$ & $324.90$          &             $315.84$ & $315.27$& $319.67$ & $314.58$                 \\
$l = 0.98$& $315.54$ & $315.26$& $468.59$ & $332.92$          &             $320.53$ & $320.44$& $325.92$ & $317.95$                 \\ \hline
\end{tabular}}
\caption{Verification results of VGG networks.}
\label{tab:VGG}
\vspace{-3em}
\end{table}
\vspace{-1em}

\subsection{Exact Analysis vs. Approximate Analysis}
We have compared our ImageStar approximate scheme with the zonotope and polytope approximation methods.
It is interesting to investigate the performance of ImageStar exact scheme in comparison with the approximate one.
To illustrate the advantages and disadvantages of the exact scheme and approximate scheme, we consider the robustness verification of VGG16 and VGG19 on a single ImageStar input set created by an adversarial attack on a bell pepper image.
The verification results are presented in Table \ref{tab:exact vs. approx}.
The table shows that for a small perturbation $\delta$, the exact and over-approximate analysis can prove the robustness of the VGG16 around some specific levels of attack in approximately one minute.
We can intuitively verify the robustness of the VGG networks via visualization of their output ranges.
An example of the output ranges of VGG19 for the case of $l = 0.95\%, \delta_{max} = 2\times 10^{-7}$ is depicted in Figure \ref{fig:vgg19-exact-range}.
One can see from the figure that the output of the index $946$ corresponding to the bell pepper label is always the maximum one compared with others, which proves that VGG19 is robust in this case.
From the table, it is interesting that VGG19 is not robust if we apply $\geq 98\%$ of the attack.
Notably, the exact analysis can give us correct answers with a counter-example set in this case.
However, the over-approximate analysis cannot prove that VGG19 is not robust since its obtained reachable set is an over-approximation of the exact one.
Therefore, it may be the case that the over-approximate reachable set violates the robustness property because of its conservativeness.
A counter-example generated by the exact analysis method is depicted in Figure \ref{fig:counter-example} in which the disturbed image is classified as strawberry instead of bell pepper since the strawberry output is larger than the bell pepper output in this case.

\vspace{-1em}
\begin{table}[h!]
    \centering
   \resizebox{0.7\linewidth}{!}{
\setlength{\arrayrulewidth}{.1em}
\begin{tabular}{llclccclcl}
\hline
\multicolumn{1}{l}{\multirow{3}{*}{$\mathbf{l}$}} & \multicolumn{1}{l}{\multirow{3}{*}{$\mathbf{\delta_{max}}$}} & \multicolumn{4}{c}{\textbf{VGG16}}                                                              & \multicolumn{4}{c}{\textbf{VGG19}}                                                              \\
\multicolumn{1}{c}{}                            & \multicolumn{1}{c}{}                                & \multicolumn{2}{c}{\textbf{Exact}} & \multicolumn{2}{c}{\textbf{Approximate}} & \multicolumn{2}{c}{\textbf{Exact}} & \multicolumn{2}{c}{\textbf{Approximate}} \\ \cline{3-10}
\multicolumn{1}{c}{}                            & \multicolumn{1}{c}{}                                & Robust       & \multicolumn{1}{c}{VT}       & Robust                 & VT                       & Robust       & \multicolumn{1}{c}{VT}       & Robust           & \multicolumn{1}{c}{VT}         \\ \hline
\multirow{2}{*}{50\%}                           & $10^{-7}$                              & Yes          & 64.56226                     & Yes                    & 60.10607                 & Yes          & 234.11977                    & Yes              & 72.08723                       \\
                                                & $2 \times 10^{-7}$                            & Yes          & 63.88826                     & Yes                    & 59.48936                 & Yes          & 1769.69313                   & Yes              & 196.93728                      \\ \hline
\multirow{2}{*}{80\%}                           & $10^{-7}$                              & Yes          & 64.92889                     & Yes                    & 60.31394                 & Yes          & 67.11730                     & Yes              & 63.33389                       \\
                                                & $2 \times 10^{-7}$                            & Yes          & 64.20910                     & Yes                    & 59.77254                 & Yes          & 174.55983                    & Yes              & 200.89500                      \\ \hline
\multirow{2}{*}{95\%}                           & $10^{-7}$                              & Yes          & 67.64783                     & Yes                    & 59.89077                 & Yes          & 73.13642                     & Yes              & 67.56389                       \\
                                                & $2 \times 10^{-7}$                            & Yes          & 63.83538                     & Yes                    & 59.23282                 & Yes          & 146.16172                    & Yes              & 121.91447                      \\ \hline
\multirow{2}{*}{97\%}                           & $10^{-7}$                              & Yes          & 64.30362                     & Yes                    & 59.79876                 & Yes          & 77.25398                     & Yes              & 64.43168                       \\
                                                & $2 \times 10^{-7}$                            & Yes          & 64.06285                     & Yes                    & 61.23296                 & Yes          & 121.70296                    & Yes              & 107.17331                      \\ \hline
\multirow{2}{*}{98\%}                           & $10^{-7}$                              & Yes          & 64.06183                     & Yes                    & 59.89959                 & No           & 67.68139                     & Unkown           & 64.47035                       \\
                                                & $2 \times 10^{-7}$                            & Yes          & 64.01997                     & Yes                    & 59.77469                 & No           & 205.00939                    & Unknown          & 107.42679                      \\ \hline
\multirow{2}{*}{98.999\%}                                    & $10^{-7}$                              & Yes          & 64.24773                     & Yes                    & 60.22833                 & No           & 71.90568                     & Unknown          & 68.25916                       \\
                                                & $2 \times 10^{-7}$                            & Yes          & 63.67108                     & Yes                    & 59.69298                 & No           & 106.84492                    & Unknown          & 101.04668                      \\ \hline
\end{tabular}}
\caption{Verification results of the VGG16 and VGG19 in which $VT$ is the verification time (in seconds) using the ImageStar exact and approximate schemes.}
\label{tab:exact vs. approx}
\vspace{-2em}
\end{table}

\begin{figure}[h!]
  \centering
      \includegraphics[width=0.7\textwidth]{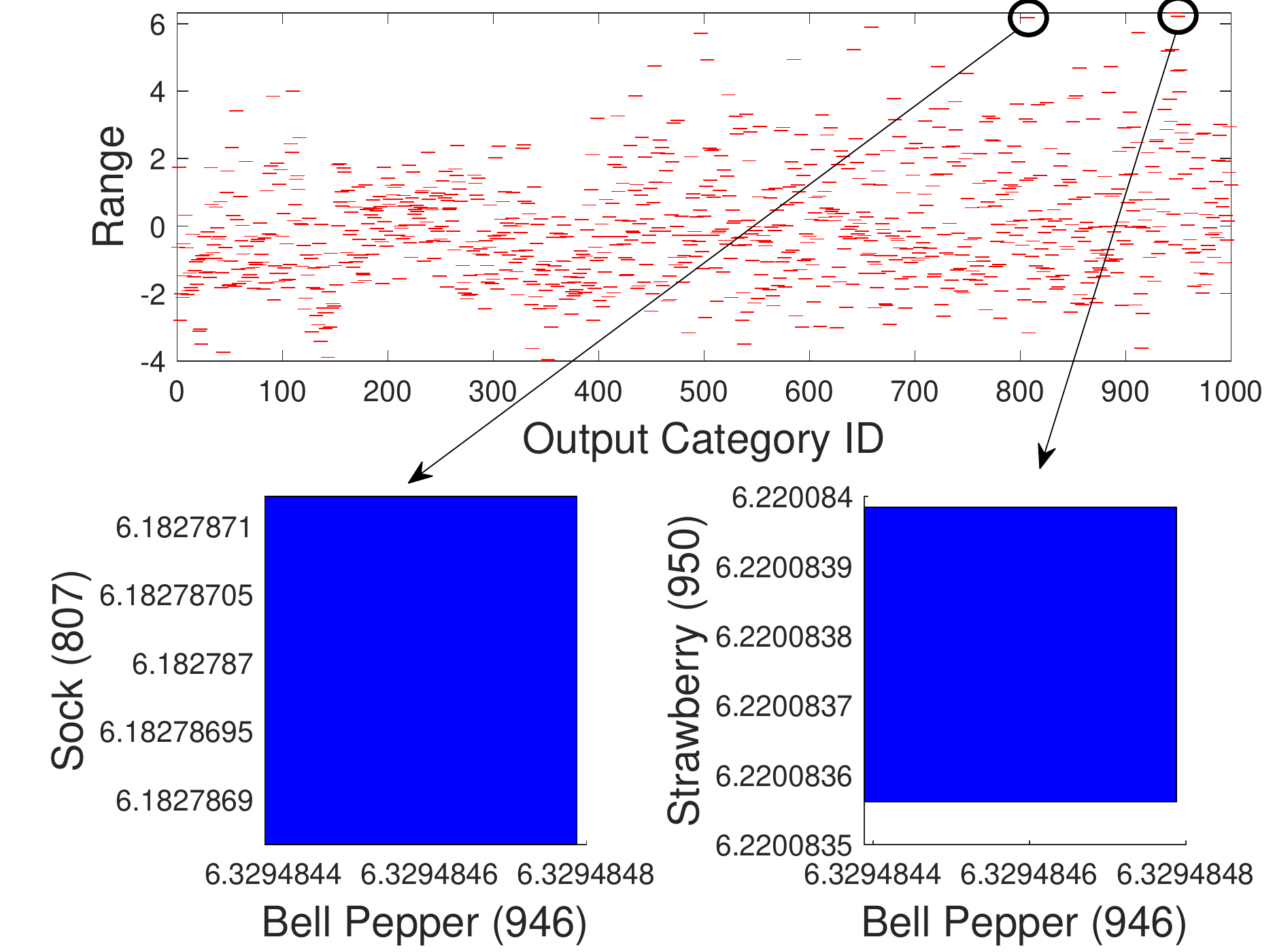}
  \caption{Exact ranges of VGG19 shows that VGG19 correctly classifies the input image as a bell pepper.}
  \label{fig:vgg19-exact-range}
  \vspace{-1em}
\end{figure}

\begin{figure}[h!]
  \centering
      \includegraphics[width=0.7\textwidth]{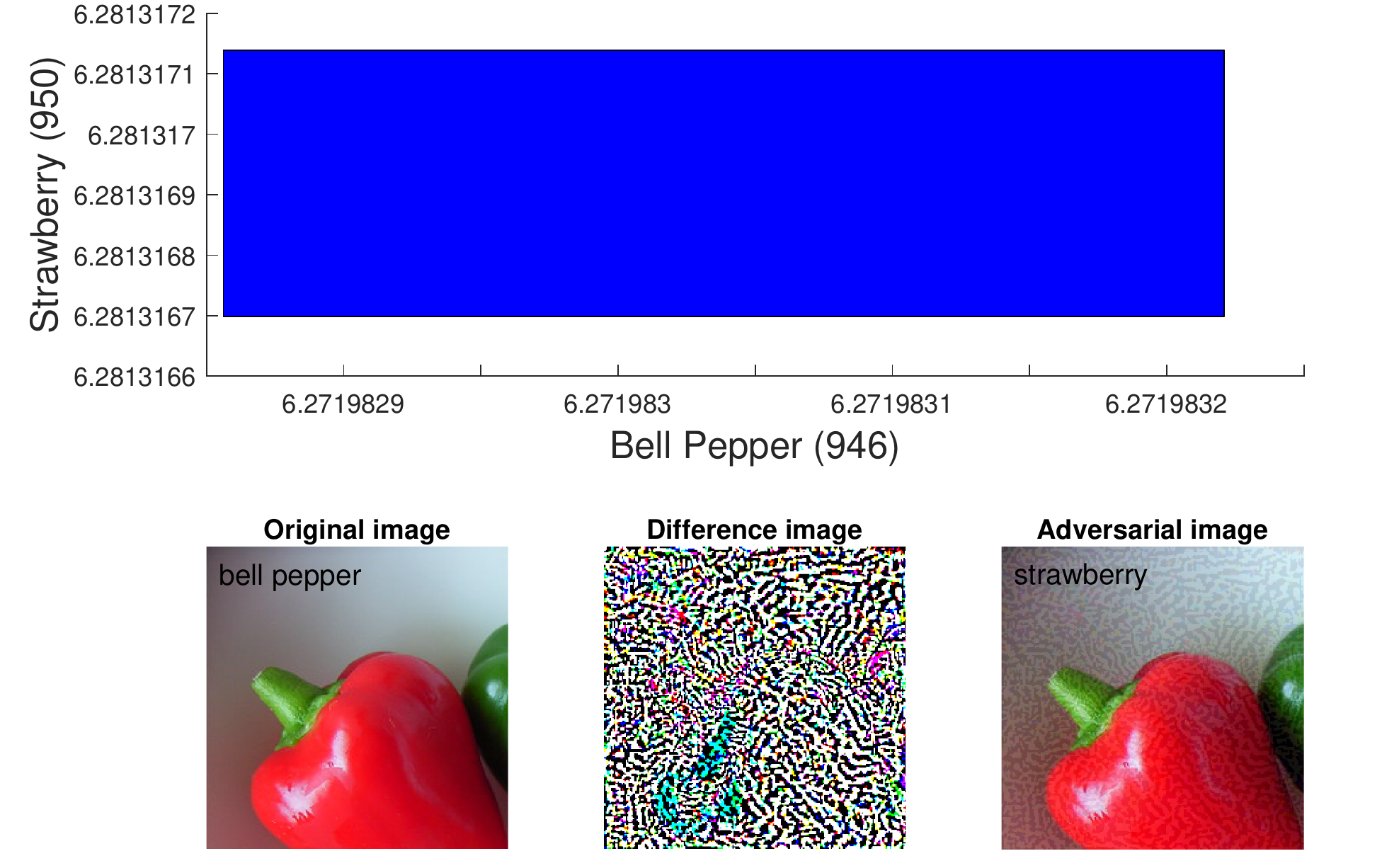}
  \caption{A counter-example shows that VGG19 misclassifies the input image as a strawberry instead of a bell pepper.}
  \label{fig:counter-example}
  \vspace{-1em}
\end{figure}

To optimize the verification time, it is important to know the times consumed by each type of layers in the reachability analysis step.
Figure \ref{fig:reachTime} described the total reachability times of the convolutional layers, fully connected layers, max pooling layers and ReLU layers in the VGG19 with $50\%$ attack and $10^{-7}$ perturbation.
As shown in the figure, the reachable set computation in the convolutional layers and fully connected layers can be done very quickly, which shows the advantages of the ImageStar data structure.
Notably, the total reachability time is dominated by the time of computing the reachable set for $5$ max pooling layers and $18$ ReLU layers.
This is because the computation in these layers concerns solving a large number of linear programing (LP) optimization problems such as finding lower bound and upper bound, and checking max point candidates.
Therefore, to optimize the computation time, we need to minimize the number of LP problems in the future.
\begin{figure}[h!]
  \centering
      \includegraphics[width=0.6\textwidth]{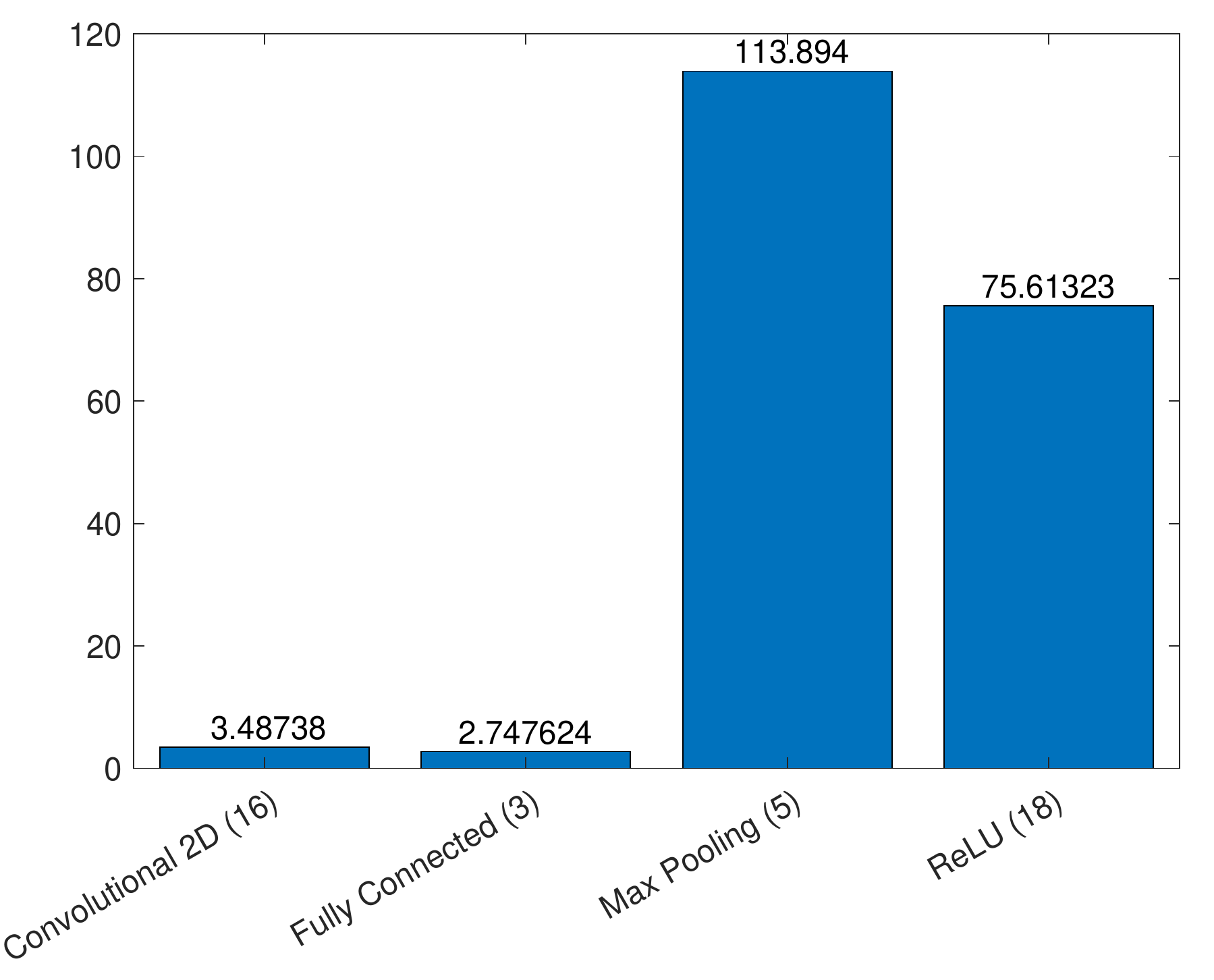}
  \caption{Total reachability time of each type of layers in the VGG19 in which the max pooling and ReLU layers dominate the total reachability time of the network.}
  \label{fig:reachTime}
  \vspace{-2em}
\end{figure}

\section{Discussion}
When we apply our approach on real-world networks, it has been shown that \textbf{the size of the input set is the most important factor that affects the performance of verification approaches}.
However, this important issue has not been emphasized in the existing literature.
Most of the existing approaches focus on the size of the network that they can analyze.
We believe that \textbf{all methods (including the method we proposed in this paper) are scalable for large networks only for small input sets}.
When the input set is large, it causes three major problems in the analysis, which are the explosions in 1) computation time; 2) memory usage; and 3) conservativeness.
In the exact analysis method, a large input set causes more splits in the max-pooling layer and the ReLU layer. A single ImageStar may split into many new ImageStars after these layers, which leads to the explosion in the number of ImageStars in the reachable set as shown in Figure \ref{fig:inputSize-effect}.
Therefore, it requires more memory to handling the new ImageStars and more time for the computation.
One may think that the over-approximate method can overcome this challenge since it obtains only one ImageStar at each layer and the cost we need to pay is only the conservativeness of the result.
The fact is, an over-approximate method usually helps reduce the computation time, as shown in the experimental results.
However, it is not necessarily efficient in terms of memory consumption.
The reason is, if there is a split, it introduces a new predicate variable and new generator. If the number of generators and the dimensions of the ImageStar are large, it requires a massive amount of memory to store the over-approximate reachable set.
For instance, \textbf{if there are $100$ splits happened in the first ReLU layer of the VGG19, the second convolutional layer will receive an ImageStar of size $224 \times 224 \times 64$ with $100$ generators. To store this ImageStar with double precision, we need approximately $2.4GB$ of memory}.
In practice, the dimensions of the ImageStars obtained in the first several convolutional layers are usually large.
Therefore, if splitting happens in these layers, we may need to deal with ``out of memory'' problem.
We see that all existing approaches such as the zonotope \cite{singh2018fast} and polytope \cite{singh2019abstract}, all face the same challenges.
Additionally, the conservativeness of an over-approximate reachable set is a crucial factor in evaluating an over-approximation approach.
Therefore, the exact analysis still plays an essential role in the analysis of neural networks since it helps to evaluate the conservativeness of the over-approximation approaches.

\begin{figure}[t]
  \centering
      \includegraphics[width=0.5\textwidth]{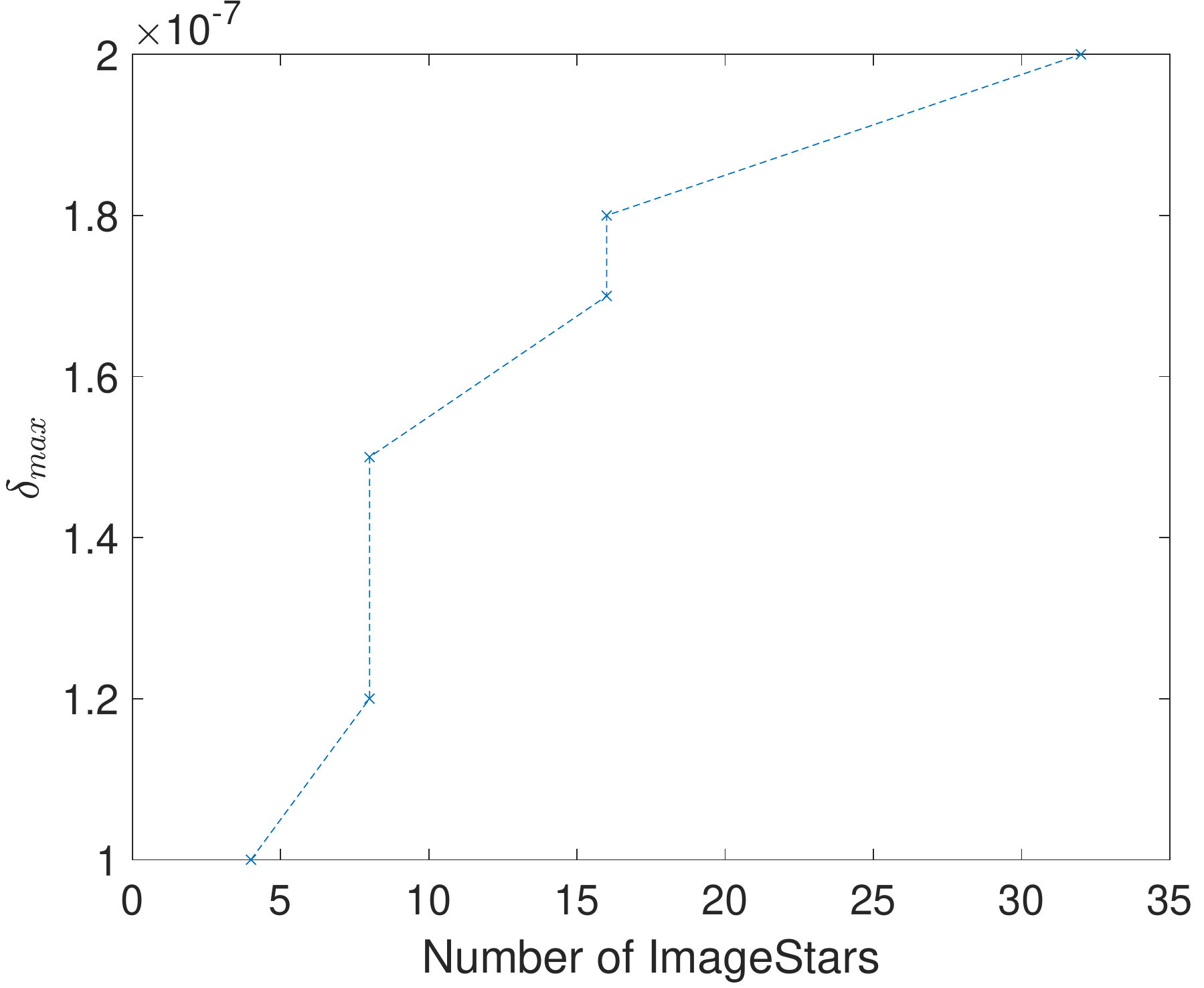}
  \caption{Number of ImageStars in exact analysis increases with input size.}
  \label{fig:inputSize-effect}
  \vspace{-2em}
\end{figure}


%

\section{Conclusion}
We have proposed a new set-based method for robustness verification of deep CNNs using the concept of the ImageStar. The core of our method are the exact and over-approximate reachability algorithms for ImageStar input sets. The experiments show that our approach is less conservative than the recent zonotope \cite{singh2018fast} and polytope \cite{singh2019abstract} approaches. It is also faster than these approaches when dealing with deep networks. Notably, our approach can be applied to verify the robustness of real-world CNNs with small perturbed input sets. It can also compute the exact reachable set and visualize the exact output range of deep CNNs, and the analysis can sped up significantly using parallel computing. We have found and shown the size of the input set to be an important factor that affects the reachability algorithms performance. Our future work is improving the proposed method to deal with larger input sets and optimizing the memory and time complexity of our computations. 


\newpage

\normalsize
\let\oldbibliography\thebibliography
\renewcommand{\thebibliography}[1]{\oldbibliography{#1}
\setlength{\itemsep}{0pt}} 
\bibliographystyle{splncs03}
\bibliography{tran,cav2020}  


\newpage

\appendix
\section{Appendix}

\subsection{Exact reachability algorithm for a max-pooling layer}\label{sec:exact-maxpool}
Algorithm \ref{alg:exact-reach-max-pooling} illustrates the exact reachability of a max-pooling layer with noticing that for an ImageStar set $I$, the anchor and generator images are put into a single 4-dimensional array $V$ called \emph{basis array} in which $V(:,:,:,1)$ is the anchor images. The algorithm works as follows. Firstly, we perform zero-padding for the ImageStar input set $I$ (line 2) by padding zeros to the anchor and generator images. Secondly, the zero-padding set $I^{\prime}$ is used to compute the size of the \emph{max map} $[h,w]$ and the start indexes $startID$ of each local region. Thirdly, we initialize the basis $V^{\prime}$ of the max map (line 6) and find all max-point candidates $maxID$ for every single point $[i, j, k]$ in the max map (line 13). If there is only one max-point candidate for a point $[i, j, k]$ in the max map, we update the basis of the max map. If not, we store this point to a list of splitting points $splitID$ (line 18) and then initialize the max map $R$ (line 19). Lastly, we perform splitting operations through the list of splitting points (line 23) to get the final output set which is an array of ImageStars. 
\begin{algorithm}[]
	\caption{Exact reachability algorithm for a max pooling layer.}\label{alg:exact-reach-max-pooling}
	\begin{algorithmic}[1]
        \footnotesize
        \Procedure{\texttt{$R = reach\_exact$}}{$I = \langle V, C, P \rangle, poolSize, stride, paddingSize$}
						\State $I^{\prime} = zero\_padding(I, paddingSize)$ ~{\scriptsize \textcolor{blue}{$\triangleright$ zero padding for the input set}}
						\State $[h,w] = get\_size\_maxMap(I^{\prime}, poolSize, stride)$ ~{\scriptsize \textcolor{blue}{$\triangleright$ compute the size of the max map}}					
						\State $nc = I^{\prime}.number\_channels$, ~{\scriptsize \textcolor{blue}{$\triangleright$ get the number of channels}}
						\State $np = I^{\prime}.number\_predicate\_variables$ ~{\scriptsize \textcolor{blue}{$\triangleright$ get the number of predicate variables}}				
						\State $V^{\prime}(:,:,nc, np+1) = zeros(h, w)$ ~{\scriptsize \textcolor{blue}{$\triangleright$ pre-allocate the basis of the max map}}										
						\State $startID = get\_startPoints(I^{\prime}, poolSize, stride)$ ~{\scriptsize \textcolor{blue}{$\triangleright$ the start index for each local region}}
						\State $maxID = cell(h,w,nc)$, ~{\scriptsize \textcolor{blue}{$\triangleright$ to store the index of the max-point candidates}}	
						\State $splitID = [~]$
						\For {\texttt{$k=1:nc$}}
								\For {\texttt{$i=1:h$}}
										\For {\texttt{$j=1:w$}}
												\State $maxID\{i, j, k\} = I^{\prime}.get\_localMax\_index(startID\{i,j\}, poolSize, k)$
												\If {$size(maxID\{i, j, k\},1) == 1$} 
														\State $[i^{\prime}~j^{\prime}~k] = maxID\{i, j, k\}$ ~{\scriptsize \textcolor{blue}{$\triangleright$ the local region has only one max-point}}
														\State $V^{\prime}(i,j,k, :) = I^{\prime}.V(i^{\prime}, j^{\prime}, k, :)$ ~{\scriptsize \textcolor{blue}{$\triangleright$ update the basis of the max map}}
												\Else
														\State $[i~j~k] \rightarrow splitID$ ~{\scriptsize \textcolor{blue}{$\triangleright$ store splitting index}}
												\EndIf
										\EndFor
								\EndFor
						\EndFor
						\State $R = \langle V^{\prime}, C, P \rangle$
						\State $n = size(splitID, 1)$ ~{\scriptsize \textcolor{blue}{$\triangleright$ number of splitting indexes}}
						\For {\texttt{$l=1:n$}}
								\State $[i~j~k] = splitID(l,:,:)$ ~{\scriptsize \textcolor{blue}{$\triangleright$ get splitting index}}
								\State $R = split(R, I^{\prime}, [i~j~k], maxID\{i, j, k\})$ ~{\scriptsize \textcolor{blue}{$\triangleright$ split ImageStars}}
						\EndFor		
				\EndProcedure 
				
				\Procedure {\texttt{$R^{\prime} = split$}}{$R,I^{\prime}, [i~j~k], maxID\{i, j, k\}$}
						\State $R = [R_1 ~R_2~\dots~R_m]$ ~{\scriptsize \textcolor{blue}{$\triangleright$ multiple input sets}}
						\State $R^{\prime} = [~]$
						\For {\texttt{$l=1:m$}}
								\State $IS = stepSplit(R_l, I^{\prime}, [i~j~k], maxID\{i, j, k\})$ ~{\scriptsize \textcolor{blue}{$\triangleright$ step splitting}}
								\State $IS \rightarrow R^{\prime}$ ~{\scriptsize \textcolor{blue}{$\triangleright$ store the new ImageStars}}
						\EndFor 
				\EndProcedure
				
				\Procedure {\texttt{$R^{\prime} = stepSplit$}}{$R_l,I^{\prime}, [i~j~k], maxID\{i, j, k\}$}
						 \State $maxID\{i, j, k\} = [i^{\prime}_1 ~ j^{\prime}_1 ~ k;~\dots~; i^{\prime}_q ~ j^{\prime}_q ~ k ]$  ~{\scriptsize \textcolor{blue}{$\triangleright$ the local region has $q$ max-points}}
						 \State $R^{\prime} = [~]$
						 \For {\texttt{$l=1:q$}} ~{\scriptsize \textcolor{blue}{$\triangleright$ a single ImageStar is split into $q$ ImageStars}}
									\State $max\_point = [i^{\prime}_l ~ j^{\prime}_l ~ k]$, $others = maxID\{i, j, k\} \setminus max\_point$ 
									\State $[C^{\prime},d^{\prime}] = getConstraints(R_l,I^{\prime}, max\_point, others)$  ~{\scriptsize \textcolor{blue}{$\triangleright$ get the constraints on the predicate variables that make a max-point candidate become the max point}}
									\State $V^{\prime} = R_l.V$, $V^{\prime}(i, j, k, :) = I^{\prime}(i^{\prime}_l,j^{\prime}_l, k,:)$  ~{\scriptsize \textcolor{blue}{$\triangleright$ update the basis}}
									\State $IS = \langle V^{\prime}, C^{\prime}, d^{\prime} \rangle$, $IS \rightarrow R^{\prime}$ ~{\scriptsize \textcolor{blue}{$\triangleright$ construct and store the reach set}}
						 \EndFor
				\EndProcedure
				
	\end{algorithmic}
\end{algorithm}

\subsection{Approximate reachability algorithm for a max-pooling layer}\label{sec:approx-maxpool}
Algorithm \ref{alg:approx-reach-max-pooling} illustrates the approximate reachability of a max-pooling layer. Similar to the exact algorithm, we perform zero-padding for the ImageStar input set $I$ (line 2) by padding zeros to the anchor and generator images. The zero-padding set $I^{\prime}$ is then used to compute the size of the \emph{max map} $[h,w]$ and the start indexes $startID$ of each local region. Thirdly, we initialize the basis $V^{\prime}$ of the max map (line 6) and find all max-point candidates $maxID$ for every single point $[i, j, k]$ in the max map (line 13). We use $np$ to count the number of predicate variables when we do over-approximate reachability (line 15). If a local region has more than one max-point candidate, we introduce a new predicate variable representing the max-point of that region. Using the max-point candidate indexes $maxID$ and the new predicate variable index $new\_pred\_id$, we update the basis of the max map (line 17-28). Finally, we add the constraints on the new introduced predicate variables (line 30-57) and then construct the over-approximate ImageStar output set $R$ (line 58). 
\begin{algorithm}[]
	\caption{Approximate reachability algorithm for a max pooling layer.}\label{alg:approx-reach-max-pooling}
	\begin{algorithmic}[1]
        \footnotesize
        \Procedure{\texttt{$R = reach\_exact$}}{$I = \langle V, C, P \rangle, poolSize, stride, paddingSize$}
						\State $I^{\prime} = zero\_padding(I, paddingSize)$ ~{\scriptsize \textcolor{blue}{$\triangleright$ zero padding for the input set}}
						\State $[h,w] = get\_size\_maxMap(I^{\prime}, poolSize, stride)$ ~{\scriptsize \textcolor{blue}{$\triangleright$ compute the size of the max map}}					
						\State $nc = I^{\prime}.number\_channels$, ~{\scriptsize \textcolor{blue}{$\triangleright$ get the number of channels}}
						\State $np_0 = I^{\prime}.number\_predicate\_variables$ ~{\scriptsize \textcolor{blue}{$\triangleright$ get the number of predicate variables}}				
						\State $V^{\prime}(:,:,nc, np+1) = zeros(h, w)$ ~{\scriptsize \textcolor{blue}{$\triangleright$ pre-allocate the basis of the max map}}										
						\State $startID = get\_startPoints(I^{\prime}, poolSize, stride)$ ~{\scriptsize \textcolor{blue}{$\triangleright$ the start index for each local region}}
						\State $maxID = cell(h,w,nc)$, ~{\scriptsize \textcolor{blue}{$\triangleright$ to store the index of the max-point candidates}}
					  \State $np = np_0$
						\For {\texttt{$k=1:nc$}}
								\For {\texttt{$i=1:h$}}
										\For {\texttt{$j=1:w$}}
												\State $maxID\{i, j, k\} = I^{\prime}.get\_localMax\_index(startID\{i,j\}, poolSize, k)$
												\If {$size(maxID\{i, j, k\},1) > 1$} 
														\State $np = np + 1$ ~{\scriptsize \textcolor{blue}{$\triangleright$ increase the number of predicate variables}}
												\EndIf
										\EndFor
								\EndFor
						\EndFor
						
						\State ~{\scriptsize \textcolor{blue}{$\triangledown$ update the basis of the max map}}
						\State $new\_pred\_id = 0$~{\scriptsize \textcolor{blue}{$\triangleright$ new predicate variable index}}
						\For {\texttt{$k=1:nc$}}
								\For {\texttt{$i=1:h$}}
										\For {\texttt{$j=1:w$}}
												\If {$size(maxID\{i, j, k\},1) == 1$} 
														\For {\texttt{$p=1:np_0 + 1$}}
																\State $[i^{\prime}~j^{\prime}~k] = maxID\{i, j, k\}$
																\State $V^{\prime}(i,j,k,p) = I^{\prime}.V(i^{\prime},j^{\prime}, k, p)$
														\EndFor
												\Else
														\State $V^{\prime}(i,j,k,1) = 0$
														\State $new\_pred\_id = new\_pred\_id + 1$
														\State $V^{\prime}(i,j,k,np_0 + 1 + new\_pre\_id) = 1$
												\EndIf
										\EndFor
								\EndFor
						\EndFor
					  \State ~{\scriptsize \textcolor{blue}{$\triangledown$ update the constraints on predicate variables}}
						\State $N = poolSize(1) \times poolSize(2)$
						\State $C^{\prime} = zeros(new\_pre\_id \times (N + 1), np)$
						\State $d^{\prime} = zeros(new\_pre\_id \times (N + 1), 1)$
						\State $l = 0$
						\State $\alpha = [\alpha_1 \dots \alpha_{np_0 + l} \dots \alpha_{np}]^T$ ~{\scriptsize \textcolor{blue}{$\triangledown$ vector of predicate variables}}
						\For {\texttt{$k=1:nc$}}
								\For {\texttt{$i=1:h$}}
										\For {\texttt{$j=1:w$}}
												\If {$size(maxID\{i, j, k\},1) > 1$} 
														\State $l = l + 1$
														\State~{\scriptsize \textcolor{blue}{$\triangledown$ get all related local pixel indexes}}
														\State $points = I^{\prime}.get\_localPoints(startID{i,j}, poolSize, k)$
														\State~{\scriptsize \textcolor{blue}{$\triangledown$ upper bound of new predicate variable: $\alpha_{np_0 + l}\leq ub \equiv C_1 \alpha \leq d_1$}}													
														\State $C_1 = zeros(1, np),~C_1(np_0 + l) = 1$
														\State~{\scriptsize \textcolor{blue}{$\triangledown$ get bounds of the local pixel values}}
														\State $[lb, ub] = I^{\prime}.get\_localBound(startID{i,j}, poolSize, k)$ 
														\State $d_1 = ub$
														\State $C_2 = zeros(N, np),~d_2 = zeros(N, 1)$
														\For {\texttt{$g=1:N$}}
																\State $[i^{\prime}~j^{\prime}] = points(g,:)$
																\State~{\scriptsize \textcolor{blue}{$\triangledown$ new constraint: $\alpha_{l}\geq x_{i^{\prime}j^{\prime}k} \equiv C_2 \alpha \leq d_2$}}	
																\State $C_2(g, 1:np_0) = I^{\prime}.V(i^{\prime}, j^{\prime}, k, 2:np_0 + 1)$,
																\State $C_2(g, np_0 + l) = -1$
																\State $d_2(g) = -I^{\prime}.V(i^{\prime}, j^{\prime}, k, 1)$
														\EndFor
														\State $C^{\prime}((l-1)\times (N+1) + 1: l \times(N+1), :) = [C_1; C_2]$
														\State $d^{\prime}((l-1)\times (N+1) + 1: l \times(N+1)) = [d_1; d_2]$
												\EndIf
										\EndFor
								\EndFor
						\EndFor
						
						\State $C^{\prime} = [I^{\prime}.C~~zeros(size(I^{\prime}.C, 1), l);~C^{\prime}]$
						\State $d^{\prime} = [I^{\prime}.d;~d^{\prime}]$
						\State $R = \langle V^{\prime}, C^{\prime}, d^{\prime} \rangle$
						
				\EndProcedure 		
				
	\end{algorithmic}
\end{algorithm}

\subsection{Architectures of MNIST networks (Figure \ref{fig:MNIST_networks})}
\begin{figure}[h!]
  \centering
      \includegraphics[width=\textwidth]{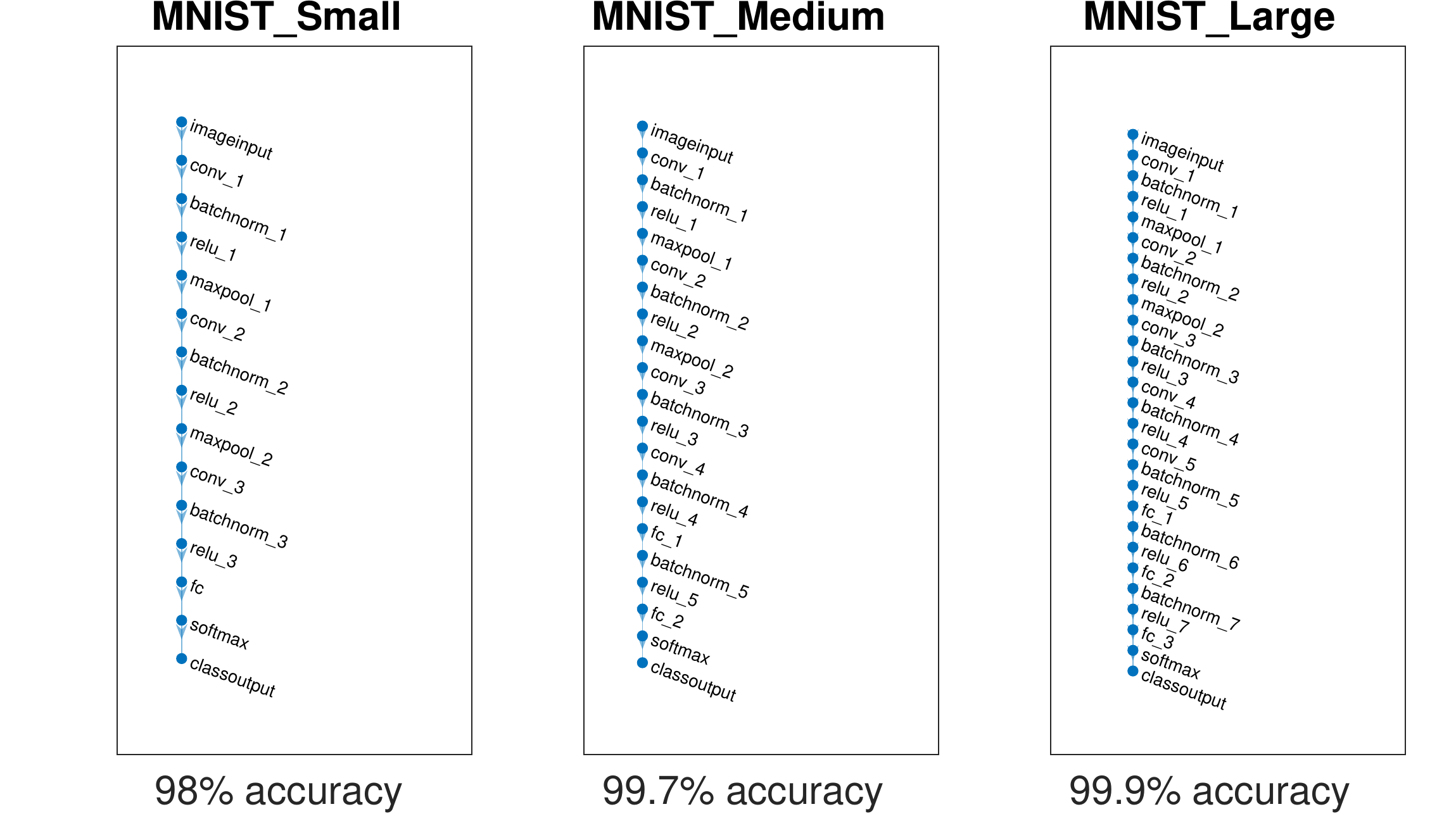}
  \caption{The architectures of the small, medium, and large MNIST classification networks.}
  \label{fig:MNIST_networks}
  \vspace{-2em}
\end{figure}

\subsection{Comparison with ERAN-DeepZ\cite{singh2018fast}}
In this section, we present the comparison between NNV and ERAN-DeepZ on the ERAN $ConvMaxPool$ network on a subset of CIFAR-10 data set containing $60000$ $32\times 32$ color images in $10$ classes.\footnote{\url{https://github.com/verivital/run_nnv_eran_comparison}}
We choose this network because it is the only one that has max-pooling layers in all ERAN networks trained on CIFAR-10. 
The network has $4$ Convolutional layers, $2$ max-pooling layer, $7$ relu layers, and $3$ fully connected layers. 
Although ERAN has chosen $100$ images for the robustness analysis, we have experienced that only $48$ images are correctly classified by the network without any attacks which means the accuracy of the network is very low ($\approx 48\%$). 
Despite the low accuracy, we still use it to clarify the advantages/disadvantages of our ImageStar approach in comparison with ERAN-DeepZ.      
We note that the pixel values in the tested images are scaled into [0,~1] in the analysis, and the well-known brightening attack \cite{gehr2018ai2} is used for the robustness analysis of the network.  
For any image, the values of some pixels $x_i$ are changed independently to $x_i^{\prime}$ under the brightening attack which results in a set of images that can be described by a zonotope $Z = \{x_i^{\prime}| 1-\delta \leq x_i \leq x_i^{\prime} \leq 1\ \vee x_i^{\prime} = x_i\}$ where $\delta$ is called a robustness bound \cite{gehr2018ai2}.

The robustness verification results are presented in Table \ref{tab:CIFAR-ConvMaxPool} which shows that our method gives a slightly better result than ERAN in terms of number of verifiable images. 
Additionally, the verification time of our approach grows very quickly as the robustness value $\delta$ increases.  
The growth of verification time is the cost we need to pay to improve the conservativeness of the over-approximate methods. 
One can see that, ERAN-DeepZ method is faster than the ImageStar method when dealing with large input sets (e.g., more attacked pixels) since it does not solve any LP optimization problems when constructing the reachable sets.
In the future, we are going to develop a more relaxed ImageStar method that can neglect to solve LP optimization to deal with larger input sets.
\begin{table}[h]
   \centering
   \resizebox{\linewidth}{!}{
\setlength{\arrayrulewidth}{.1em}
\begin{tabular}{cccccccccc}
        \hline
        \multicolumn{10}{c}{\textbf{Robustness Results (in Percent)}}                                          \\ \hline
        \multicolumn{2}{c}{$\delta = 0.005$} & \multicolumn{2}{c}{$\delta = 0.008$} & \multicolumn{2}{c}{$\delta = 0.01$} & \multicolumn{2}{c}{$\delta = 0.012$} & \multicolumn{2}{c}{$\delta = 0.015$} \\  \hline
        \textit{ERAN} & \textit{ImageStar~~}       & \textit{~~ERAN}       & \textit{ImageStar~~}       & \textit{~~ERAN}        & \textit{ImageStar}  & \textit{~~ERAN}        & \textit{ImageStar} & \textit{~~ERAN}        & \textit{ImageStar}     \\ \hline
				$48/48$ & $48/48$& $46/48$ & $47/48$& $46/48$ & $47/48$ & $44/48$ & $46/48$ & $44/48$ & $46/48$\\ \hline
       \multicolumn{10}{c}{\textbf{Verification Times (in Seconds)}}                                             \\ \hline
				$115$ & $78.05$& $122$ & $155.45$& $123$ & $153.67$ & $130$ & $454.44$ & $129$ & $468.49$\\ \hline
\end{tabular}}
\caption{Verification results of the CIFAR-10 $ConvMaxPool$ network.}
\label{tab:CIFAR-ConvMaxPool}
\end{table}


%






\end{document}